\documentclass[11pt]{article}

\usepackage[a4paper,margin=2.5cm]{geometry}
\usepackage{amsmath,amssymb,amsthm}
\newtheorem{theorem}{Theorem}[section]
\newtheorem{lemma}[theorem]{Lemma}
\newtheorem{definition}[theorem]{Definition}
\newtheorem{proposition}[theorem]{Proposition}

\usepackage{bm}
\usepackage{hyperref}
\usepackage{tikz}
\usepackage{pgfplots}
\pgfplotsset{compat=1.18}
\usepgfplotslibrary{fillbetween}
\usetikzlibrary{intersections}

\title{A Minimal $\kappa$--$\tau$ Logic for Risk-Sensitive Abduction}

\author{Remo Pareschi\\
STAKE Lab, University of Molise\\
86100 Campobasso, Italy\\
\texttt{remo.pareschi@unimol.it}}

\date{}

\begin{document}

\maketitle

\begin{abstract}
Standard approaches to abductive reasoning can retain multiple
candidate explanations, but they do not generally combine explicit
compositional cross-hypothesis interaction with an internal,
rival-sensitive commitment judgment.
This paper argues that in risk-sensitive domains---where premature
commitment carries asymmetric downside costs---the timing of
commitment is itself a governed decision that the inferential
apparatus should formally represent.

We present a minimal $\kappa$--$\tau$ logical framework built on
two primitives: epistemic interaction among hypotheses ($\kappa$)
and a normative commitment threshold ($\tau$). The resulting
system models pre-collapse plausibility, interaction-dependent
synthesis, and threshold-based commitment. Hypotheses may coexist,
reinforce or inhibit one another, and form emergent composite
explanations, while collapse into committed conclusions is
regulated by governance constraints rather than forced by inference
alone. The logic is developed in two complementary modes sharing the
interaction relation and the governance apparatus: a \emph{synthetic} mode, in which
atomic hypotheses are composed upward into emergent explanations, and
an \emph{analytic} mode, in which complex observed states of affairs
are decomposed into causal clusters of latent factors, with
inter-cluster interaction and threshold-governed commitment operating
at both the cluster and the factor level. The framework provides
formal machinery for domains in which the distinction between
\emph{highly likely} and \emph{commit-worthy} is operationally
consequential---where premature collapse of uncertainty carries risks
that the logic itself is designed to prevent.

The $\kappa$--$\tau$ logic is positioned as the symbolic governance
layer of a neurosymbolic architecture: its epistemic parameters
(hypothesis weights, interaction coefficients) are naturally estimated
by neural components---semantic embeddings and generative models, as
demonstrated in existing computational realizations---while its
normative parameters (commitment threshold, activation floor,
rivalry margin) remain
under explicit human governance. The interface between estimable
epistemic structure and regulated commitment is thereby a formal
feature of the logic itself, yielding transparent and auditable abductive
reasoning over neurally supplied content.
\end{abstract}

\section{Introduction}
\label{sec:intro}
Abductive reasoning---the mode of inference that Peirce
\cite{peirce1935abduction} identified as the source of explanatory
hypotheses---concerns the formation of plausible explanations
under incomplete, ambiguous, or contradictory information. Standard
approaches to abduction---whether logic-based, probabilistic, or
set-covering---can and often do retain multiple explanations: logical
systems may return several abductive solutions, Bayesian systems
maintain posterior distributions, and set-covering explanations are
characteristically multi-factor sets. What they do not generally
provide is the combination developed here: explicit compositional
interaction \emph{between} retained hypotheses, and an internal,
rival-sensitive judgment governing \emph{when} the reasoning process
is entitled to commit.

This paper proposes an alternative perspective in which
uncertainty is not merely a transitional inconvenience but a structured
epistemic state. Hypotheses are allowed to coexist, interact, reinforce,
or inhibit one another. Collapse into committed conclusions is governed
by decision constraints rather than enforced by inference alone---particularly
critical in risk-sensitive contexts where premature commitment carries
high downside costs.

The central observation motivating this work is that, across
classical and probabilistic approaches to abduction, what is
typically external to the abductive inference itself is a distinct,
interaction-sensitive criterion governing \emph{when} a plausible
explanation becomes actionable: acceptance is delegated to a
selection rule applied after inference, rather than represented
within it. In many real-world reasoning contexts,
however, agents may recognize hypotheses as plausible while postponing
commitment to them due to risk, conflict, or insufficient explanatory
integration. Commitment may therefore be deferred until evidence favours
one direction over another, potentially enabling reconciliation between
hypotheses initially perceived as heterogeneous or even contradictory.

The framework developed here formalizes this distinction through two
complementary mechanisms, captured by two parameters:

\begin{itemize}
    \item $\kappa$: an epistemic interaction operator that captures
    constructive and destructive relations among hypotheses, allowing
    explanatory structures to reinforce or inhibit one another within
    a shared epistemic state
    \item $\tau$: a normative commitment threshold that separates
    epistemic activation from actionable acceptance, determining when
    hypotheses move from plausibility to commitment
\end{itemize}

In this sense, the logic can be interpreted as a formal abstraction of
\emph{governed abduction}---abductive reasoning in which the timing of
commitment is itself a regulated decision. It is related to threshold
models of belief, non-monotonic reasoning, and formal accounts of
explanation, while extending these traditions with an explicit semantics
of epistemic interaction among competing explanations.

The framework did not arise in a vacuum. It distils the governance
mechanics of the Quantum Abduction framework \cite{pare1}, in which
hypothesis coexistence, interference, and governed collapse were
implemented computationally through semantic embeddings and generative
AI and exercised on forensic, clinical, and historical case studies,
building in turn on an empirical demonstration of interaction among
competing heuristics in strategic reasoning
\cite{ghisellini2025entangled}. The programme has since been extended
on the analytic side: companion work on Analytic Abduction
\cite{pareschi2026analytic} reverses the compositional orientation,
decomposing complex observed states of affairs into causal clusters
of latent factors under the same governance discipline, with
particular attention to the legibility of suspended decomposition as
a coordination object between human and artificial agents. The
present paper isolates the logical kernel of that entire apparatus: a
minimal semantics in which the interaction and commitment structure
can be studied on its own terms, independently of any particular
computational realization. It develops the logic first in its
\emph{synthetic} mode---composing atomic hypotheses upward into
emergent explanations---and then in its \emph{analytic} mode,
formalizing causal decomposition, inter-cluster interaction, and
two-level commitment within the same semantic framework
(Section~\ref{sec:analytic}), thereby discharging what earlier
formulations of the logic left as future work.

This substantive positioning naturally fits within a neurosymbolic
perspective. The
$\kappa$--$\tau$ logic is, by design, the symbolic half of a
neurosymbolic division of labour: its \emph{epistemic} parameters---the
weight function $w$ and the interaction structure $\kappa$---are
quantities that existing implementations already estimate from the
geometry of neural sentence embeddings, with generative models
proposing the hypotheses and syntheses over which the logic operates
\cite{pare1,pareschi2026analytic}; its \emph{normative} parameters---the
commitment threshold $\tau$ and the activation floor
$\epsilon$---are not learned but set by governance. The boundary
between what may be learned and what must be governed is thus drawn
at the $\kappa$/$\tau$ interface (with the provenance
qualifications made explicit in Section~\ref{sec:neurosymbolic}), and
Section~\ref{sec:neurosymbolic} makes this architectural reading
explicit.

The framework introduced here should therefore be understood as a
minimal logical kernel whose semantic structure opens the way to
further meta-theoretic analysis, including axiomatization, structural
properties of hypothesis interaction, and extensions to richer
abductive reasoning settings.

The rest of this paper is structured as follows. Section~2 introduces
the formal language and the notion of abductive states that constitute
the basic epistemic setting of the framework. Section~3 defines the
scoring semantics that evaluates hypotheses and composite explanations
within an abductive state. Section~4 presents the main inference
principles governing plausibility and commitment. Section~5 establishes
a basic structural property of the framework, clarifying how interaction
among hypotheses affects their epistemic status. Section~6 studies
derivation and closure conditions for abductive reasoning within the
system. Section~7 provides an algebraic perspective on the logic and its
interaction structure. Section~8 provides an illustrative example
of suspended abductive reasoning in a crisis-analysis setting.
Section~9 extends the logic from the synthetic to the analytic mode,
formalizing governed causal decomposition and continuing the
crisis-analysis example in the analytic direction. Section~10
articulates the neurosymbolic realization of the framework.
Section~11 situates the framework relative to related work, while
Section~12 outlines directions for further development. Section~13
concludes.

\section{Language and Abductive States}
\label{sec:language}
\subsection{Language}

Let $\mathsf{At}$ be a set of atomic propositions partitioned into
hypotheses $\mathsf{H}$ and (optionally) observations $\mathsf{O}$.
The language is two-sorted, separating \emph{content formulas},
which receive scores, from \emph{judgments}, which are evaluated by
satisfaction:

\[
t ::= h \mid (t \otimes t)
\qquad (h \in \mathsf{H})
\tag{hypothesis terms}
\]

\[
\pi ::= a \mid (\pi \otimes \pi) \mid (\pi \wedge \pi)
\mid (\pi \vee \pi)
\qquad (a \in \mathsf{At})
\tag{positive content formulas}
\]

\[
\chi ::= \pi \mid \neg \chi \mid (\chi \wedge \chi)
\mid (\chi \vee \chi)
\tag{content formulas}
\]

\[
J ::= P\,\chi \mid C_{\tau}\,\chi
\tag{judgments}
\]

\noindent where:

\begin{itemize}
    \item $P\chi$: the judgment that $\chi$ is epistemically plausible
    \item $C_{\tau}\chi$: the judgment that $\chi$ is commit-worthy
    \item $\otimes$: synthesis operator, admitted exactly on positive
    (negation-free, judgment-free) content formulas, with hypothesis
    terms as its distinguished fragment
\end{itemize}

The typing is deliberate and enforced by the grammar. Synthesis
composes explanatory content, so $\otimes$ is confined to positive
content formulas: the complement of a hypothesis is not itself an
explanatory structure, and judgments are not content. Judgments do
not nest ($P(P\chi)$ is not well formed) and are not score-bearing:
the scoring semantics of Section~\ref{sec:scoring-semantics} is
defined on content formulas only, while judgments are evaluated by
threshold satisfaction conditions over content scores. Nested or
score-bearing judgments are excluded from the minimal system, not
because they are incoherent, but because admitting them would
require an additional layer of semantics that the present kernel
does not need.

Thus, hypothesis terms are inductively generated from atomic hypotheses;
therefore, for every term $t$, $\mathrm{At}(t)\neq\varnothing$.
The framework remains neutral with respect to the internal
representation of atomic hypotheses. Elements of $\mathsf{H}$ may
correspond to natural-language explanatory statements, symbolic
propositions, or other representational units produced by external
reasoning systems. The logic operates over these hypotheses as abstract explanatory
tokens, requiring only that plausibility weights and interaction
relations be defined for them. In this
sense, the framework follows the spirit of abstract logical systems
in which reasoning operates over explanatory tokens independently
of their internal linguistic structure.

Hypothesis terms, therefore, represent compositional explanatory
structures, while the surrounding logical language allows
their epistemic evaluation through the plausibility and
commitment operators.

The operators $P$ and $C_\tau$ are not modalities in the standard
Kripkean sense, nor graded modalities: they carry no accessibility
semantics and their arguments do not embed. They are best understood
as \emph{threshold judgment operators}---valuation tests expressing
two distinct consequence regimes induced by the scoring semantics:
plausibility consequence and commitment consequence.

\subsection{Abductive Models and States}

The semantics separates the fixed parameters of a reasoning scenario
from its mutable epistemic content. An \emph{abductive model} is a
structure

\[
M = \langle \mathsf{H}, \mathsf{O}, \kappa, \kappa_o, \sigma,
\lambda, \epsilon, \tau, \Pi \rangle
\]

\noindent where:

\begin{itemize}
    \item $\mathsf{H}$, $\mathsf{O}$: hypotheses and observations
    \item $\kappa: \mathsf{H} \times \mathsf{H} \rightarrow [-1,1]$:
    epistemic interaction
    \item $\kappa_o: \mathsf{O} \times \mathsf{H} \rightarrow [-1,1]$:
    evidence compatibility, the channel through which observations
    support or undercut hypotheses (used by the dynamics of
    Section~\ref{sec:derivation} and the analytic mode of
    Section~\ref{sec:analytic})
    \item $\sigma: \mathsf{O} \rightarrow [0,1]$: observation
    reliability
    \item $\lambda \in (0,1]$: interaction-scaling parameter
    \item $\epsilon \in (0,1)$: activation floor
    \item $\tau \in (0,1]$: normative commitment threshold, subject
    to the constraint $\epsilon < \tau$
    \item $\Pi$: a governance policy regulating competition and
    pre-emption among rival conclusions
    (Sections~\ref{sec:derivation} and~\ref{sec:analytic}); in the
    minimal system, $\Pi$ is determined by a margin function $\delta$
    introduced in Section~\ref{subsec:governed-competition}
\end{itemize}

An \emph{abductive state} over $M$ is a pair

\[
S = \langle M, w \rangle,
\qquad
w: \mathsf{H} \rightarrow [0,1],
\]

\noindent where $w$ assigns each hypothesis its current plausibility
weight. States over a common model differ only in their weight
function, which is exactly the shape of the evidence dynamics
studied in Section~\ref{sec:derivation}. Where no confusion can
arise, we abbreviate the state by the components in active use,
writing $S = \langle \mathsf{H}, w, \kappa, \tau, \epsilon \rangle$
for discussions in which the evidential channel and the governance
policy are not in play.

\noindent The constraint $\epsilon < \tau$ guarantees that the
suspension interval $[\epsilon,\tau)$ is nonempty: there is always
epistemic room between activation and commitment.

\noindent Weights represent pre-collapse plausibility rather than
probability. The interaction function $\kappa$ is assumed to be bounded
in the interval $[-1,1]$, where positive values represent constructive
interaction between hypotheses and negative values represent inhibitory
interaction. The sign of $\kappa$ measures \emph{explanatory
compatibility}---whether two hypotheses cohere as parts of one
account---not the causal sign of an influence: a mechanism in which one
factor causally suppresses another may still be explanatorily coherent,
and would carry positive $\kappa$ on this reading. Compatibility so
understood is mutual, and we accordingly assume $\kappa$
\emph{symmetric}: $\kappa(h_1,h_2)=\kappa(h_2,h_1)$. (A directed
variant, in which $\kappa$ encodes asymmetric influence rather than
compatibility, is expressively attractive for causal modelling and is
registered among the refinements of
Section~\ref{sec:future-directions}; the minimal framework does not require it.)
No transitivity assumption is imposed, allowing interaction patterns to
reflect domain-specific epistemic relations. We stipulate
$\kappa(h,h)=0$ for every $h\in\mathsf{H}$: a hypothesis does not
interact with itself.

The null-diagonal convention lifts to composite terms through a
\emph{structural equivalence}. Let $\equiv_\otimes$ be the smallest
congruence on hypothesis terms generated by commutativity,
$t \otimes u \equiv_\otimes u \otimes t$. Since $\kappa$ is symmetric,
all score-relevant quantities are invariant under $\equiv_\otimes$
(Section~\ref{sec:algebraic-semantics}), so $\equiv_\otimes$-equivalent
terms carry the same explanatory content. We stipulate
$\kappa^*(t_1,t_2)=0$ whenever $t_1 \equiv_\otimes t_2$ for the lifted
interaction of Section~\ref{sec:scoring-semantics}: repeating a
hypothesis---atomic or composite, in either order---cannot inflate its
score. The quotient is deliberately \emph{not} taken by associativity:
Section~\ref{sec:algebraic-semantics} proves that $\otimes$ is
non-associative, so differently bracketed terms over the same atoms are
genuinely distinct explanatory structures and may legitimately
interact.

The parameter $\tau$ specifies the commitment boundary separating
hypotheses that remain merely plausible from those that become
decision-eligible. In this sense, $\tau$ regulates the transition from
epistemic activation to normative commitment within an abductive state.
The threshold $\tau$ is a parameter of the model $M$ carried by the
state, not an index on
the language: the operator $C_\tau$ is evaluated against the threshold
carried by $S$, and Threshold Monotonicity (Section~\ref{sec:inference})
is a comparison between two models identical except in their threshold.
No family of operators over a continuum of thresholds is assumed.
Together, the activation floor $\epsilon$, the commitment threshold
$\tau$, and the margin function carried by $\Pi$ constitute the
governance parameters of the model; richer governance policies---%
term-specific thresholds, richer pre-emption
rules among competitors---are refinements deferred to future work.

Conceptually, an abductive state can be viewed as defining an interaction
space among hypotheses shaped by $\kappa$, within which plausibility
evolves until crossing the commitment boundary determined by $\tau$.

The following subsection clarifies this distinction, which is structurally
consequential for the entire framework.

\subsection{Plausibility vs.\ Probability}

Plausibility and probability differ in normative structure,
compositional behaviour, and inferential role. These differences
become consequential precisely in the situations that motivate
the $\kappa$--$\tau$ framework: reasoning under ambiguity, contradiction, and
sustained uncertainty, particularly when decisions carry
\emph{asymmetric risk}---where the cost of premature commitment
vastly exceeds the cost of continued investigation.

\paragraph{An intuitive illustration.}
Consider a patient presenting with an acute syndrome compatible with
two distinct underlying conditions, $H_1$ and $H_2$, whose first-line
treatments are mutually exclusive: administering the therapy appropriate
to one is actively harmful if the other obtains. Early clinical evidence
is ambiguous: some findings point towards $H_1$, others towards $H_2$, and
several are compatible with both.

Under a probabilistic framework, the clinician assigns
$P(H_1) + P(H_2) \le 1$, and as evidence accumulates, one
probability rises while the other falls. The two hypotheses are
locked in a zero-sum competition. If $P(H_1) = 0.55$ and
$P(H_2) = 0.45$, MAP selection---or any symmetric-loss decision
rule---favours $H_1$. But
acting on that assessment alone---administering only the
treatment for $H_1$---could prove fatal if the true condition
is $H_2$.

The risk-managing response, where it is clinically feasible, is to
maintain both hypotheses as simultaneously active, prepare or
administer treatment for both in so far as they are compatible, and
wait for a confirmatory test that will eventually force resolution.
During this interval, the two hypotheses are not in zero-sum
competition. They \emph{interact}: evidence that partially
supports $H_1$ also constrains what $H_2$ would need to look like
if it were true, and vice versa. The clinical team reasons
\emph{across} the hypotheses, not merely \emph{between} them.
The tension between $H_1$ and $H_2$ is not a defect to be
eliminated, but an informative structure that guides resource
allocation---which tests to order, which treatments to prepare,
when to commit. Refusing to collapse prematurely is itself a
\emph{risk-management} response: it hedges against the catastrophic
downside of committing to the wrong condition.

From the perspective of the $\kappa$--$\tau$ logic, this configuration
is exactly a state of \emph{suspended derivation}: both hypotheses
have high plausibility scores $\mathrm{sc}_S(H_i)$, so that
$\vdash_S^p H_1$ and $\vdash_S^p H_2$, but neither crosses the
commitment threshold $\tau$, so collapse derivation
$\vdash_S^c H_i$ does not yet occur. The point of the example is not
its clinical novelty, but the clean separation it exhibits between
a strong inferential state and the absence of commitment.

\paragraph{A contrasting schema: irreversible action before verification.}
The diagnostic scenario illustrates a case in which suspended
derivation is \emph{correctly} maintained. The complementary failure
mode arises when a high plausibility score is treated as if it were
sufficient for commitment. Its schema is common to many
protocol-governed domains: an irreversible step $A$ may be taken, by
protocol, only after a condition $V$ has been \emph{verified}; $V$
has a high prior of holding, and verification costs time. Familiar
instances include destroying the sole physical copy of a record
before the archival copy has been confirmed readable, irreversibly
committing an industrial batch process before the confirmatory assay
of a critical input, or demolishing a structure before the final
clearance inspection. In each instance, a purely probabilistic
assessment supports proceeding: the prior on $V$ is high, delay is
costly, and the base rates are well documented. But the risk profile
is maximally asymmetric---the upside of proceeding early is modest
(time saved), while the downside if $V$ fails is catastrophic and
irrecoverable---and the verification requirement is precisely the
normative safeguard that the protocol encodes against this asymmetry.
In the vocabulary of the $\kappa$--$\tau$ logic, the agent who
proceeds early treats $\vdash_S^p V$ (high plausibility) as though it
were $\vdash_S^c V$ (commit-worthiness), taking the irreversible step
without the evidence having crossed the governance threshold $\tau$.
The contrast with the diagnostic case is structurally exact: there,
suspended derivation is maintained, hypotheses interact, and action
is kept reversible as far as possible---a risk-hedging posture
encoded directly in the logic; here, an irreversible step is taken on
the basis of a merely plausible condition---a risk-accepting posture
that the governance layer is designed to block. The difference lies
in the presence or absence of the derivation/commitment distinction
that the $\kappa$--$\tau$ logic internalizes.

This contrast also delimits the scope of the claim. Probabilistic
reasoning is perfectly adequate—indeed optimal—in low-stakes,
high-volume decision environments such as recommendation systems,
spam filters, or routine risk stratification, where individual
errors are reversible or tolerable and expected-value optimization
over many trials is the appropriate objective. The limitations of
pure probabilism emerge specifically in high-stakes, low-frequency,
irreversibility-sensitive contexts, where the distinction between
“highly likely” and “commit-worthy” becomes operationally
consequential. In the language of risk theory, these are precisely
the environments characterized by \emph{tail risk}---where the
probability of catastrophic outcomes is low but the magnitude
of loss is extreme and where the asymmetry between upside and
downside renders expected-value reasoning insufficient.
In these contexts, normative plausibility, with its
explicit separation of inference from commitment, provides the
appropriate decision-theoretic foundation.

\paragraph{Anticipating an objection.}
From a subjective Bayesian standpoint in the tradition of de Finetti
\cite{definetti1974theory}, probabilities represent coherent betting
rates, and decisions follow by maximizing expected utility
\cite{savage1954foundations,berger1985statistical}. Thus, a natural
response from this perspective is that the verification failure just
schematized could
have been averted within standard Bayesian decision theory by
incorporating risk via an expected utility calculation with an
appropriate loss function, which would have recommended against
the irreversible step before verification, given the catastrophic
cost of $V$ failing. This response is technically correct,
and the present argument does not claim otherwise.

However, two structural observations qualify it. First, the
Bayesian apparatus handles this case through an \emph{external}
decision-theoretic supplement---utility functions, loss
asymmetries, risk aversion parameters---that is not part of
the probabilistic inference itself. The inferential output of
Bayesian reasoning is the posterior distribution; the question
of whether to \emph{act} on that posterior is answered by a
separate normative layer. An agent who computes
$P(V) = 0.97$ has correctly completed the
probabilistic inference. The error occurs in the transition
from inference to action---a transition that probability
theory does not internally regulate. The $\kappa$--$\tau$ logic, by contrast,
\emph{internalizes} this distinction: the threshold $\tau$ is
a parameter of the logic, and the pre-collapse state $\vdash^p$
is a formal inferential output that explicitly represents
justified non-commitment. The logic itself structurally blocks
premature collapse, rather than relying on external normative
discipline to prevent it.

Second, and more fundamentally, expected utility calculations
presuppose that the relevant probabilities and loss magnitudes
are \emph{available}---that the decision-maker can assign
meaningful quantitative values to the likelihood of each
outcome and the cost of each error. In protocol-governed
verification cases, this assumption is arguably met: base rates
for the verified condition are typically well-documented, and the
cost of failure is clearly
catastrophic. But in the more characteristic settings for which
the framework is designed---crisis management under streaming,
heterogeneous, and partially contradictory evidence; evolving
hypothesis spaces in which new explanatory constructs emerge
during the investigation; genuinely novel situations lacking
stable base rates---this assumption breaks down.
Under such conditions, which decision theory recognizes as
\emph{Knightian uncertainty} \cite{knight1921risk}, risk cannot be meaningfully
quantified---Knight's own distinction between ``risk'' (quantifiable)
and ``uncertainty'' (not quantifiable) is precisely at issue. Single-prior Bayesian modelling may then require priors whose
justification becomes contestable under deep uncertainty. The framework's response is structurally different:
it maintains hypotheses in suspension, allows interaction to
generate emergent explanations, and defers commitment until
governance conditions are met---without pretending to quantify
what cannot be quantified. Suspended derivation provides a
principled inferential posture precisely where expected utility
calculations lack reliable inputs.

\paragraph{Imprecise probability and ambiguity aversion.}
A more sophisticated response from within the probabilistic tradition
replaces the single prior with a \emph{set} of priors. Gilboa and
Schmeidler's maxmin expected utility \cite{gilboa1989maxmin} and
Walley's imprecise probabilities \cite{walley1991statistical} both
formalize justified non-commitment by generalizing the credal state
itself, and both were developed precisely in response to Knightian
uncertainty and the Ellsberg-style phenomena that single-prior
Bayesianism cannot accommodate. These frameworks share the present
paper's diagnosis, but they differ from the $\kappa$--$\tau$ logic on
two structural counts. First, they generalize the belief state while
leaving the decision rule external: maxmin expected utility still
selects an act by optimizing against the least favourable prior in
the set, so the transition from credal state to action remains
governed by a decision-theoretic supplement rather than by the
inferential apparatus itself. The $\kappa$--$\tau$ logic instead
internalizes the commitment boundary as a parameter of the epistemic
state, with $\vdash^p$ a formal output representing sustained
non-commitment. Second, credal sets have no counterpart of the
interaction operator: hypotheses within a set of priors are
alternatives over which imprecision is expressed, not explanatory
structures that reinforce or inhibit one another and synthesize into
emergent composites. Imprecise probability widens the
\emph{representation} of uncertainty; it does not make hypothesis
interaction compositional.

\paragraph{Structural contrasts.}

Four properties distinguish the framework's plausibility from probability.

\emph{(i) Normative constraints.}
Probabilities are governed by Kolmogorov's axioms: they must
sum to one over a partition, compose via the law of total
probability, and update through conditionalization. The framework's plausibility
weights carry none of these structural commitments. Weights are
individually assigned in $[0,1]$ and need not sum to one across
hypotheses---indeed, in typical abductive states they do not.
Updating is governed by $\kappa$-modulated evidence projection
rather than by Bayes' rule. This reflects the fact that plausibility
tracks \emph{explanatory strength relative to evidence} rather than
\emph{degree of belief under a coherent partition}.

In the clinical example, this means both $H_1$ and $H_2$ can
simultaneously carry high plausibility without contradiction.
Maintaining both as active is representable as a legitimate
epistemic state, not as an approximation to an underlying
probability that the clinician has failed to resolve.

\emph{(ii) Interaction.}
In probability theory, the composition of events is governed
by set-theoretic overlap: $P(A \cup B) = P(A) + P(B) - P(A \cap B)$.
The ``interaction'' term is purely extensional. Probabilistic models can of course carry interaction terms and
structured dependencies; the point is narrower: ordinary event
composition does not itself provide a $\kappa$-indexed explanatory
synthesis operation. In the present framework, the
synthesis operator introduces exactly that: $\kappa$-dependent
interference---constructive or destructive---as part of the
composition itself.
Two hypotheses can \emph{amplify} each other's plausibility through
positive interaction, or \emph{attenuate} it through negative
interaction. The lattice join is the zero-interaction special case
of synthesis (Theorem~\ref{thm:kappa-free-reduction}): absent
interaction, combining two explanations yields no more than the
stronger of the two, and any gain above that baseline is produced
solely by constructive interaction (the magnitude of the deviation
is bounded in Theorem~\ref{thm:bounded-perturbation}). There is no
additive carry-over term as in probabilistic aggregation.

In the clinical scenario, the mutual exclusivity of $H_1$ and
$H_2$ ($\kappa < 0$) attenuates their synthesis but does not force
premature elimination. Meanwhile, if a third hypothesis were
compatible with one of them ($\kappa > 0$), their combination
could yield an emergent composite explanation. Probability theory
can, of course, \emph{represent} such a composite---by enlarging the
event algebra with a new hypothesis and assigning it a probability;
what it lacks is an operation that \emph{generates} the composite's
epistemic standing from the components' interaction. In the present
framework that generation is the synthesis clause itself.

\emph{(iii) Closure and commitment.}
In standard Bayesian practice, the posterior distribution is the
terminal \emph{inferential} output: once it is obtained, what remains
is decision, delegated to a separate layer that maximizes expected
utility or selects the maximum a posteriori hypothesis. In the
present framework, having a
plausibility distribution is explicitly \emph{not} closure. The
pre-collapse state is a legitimate inferential output.
Commitment requires an additional normative act---crossing the
threshold $\tau$ and, under the governance policy, clearing the
rivalry structure. This decoupling of inference from decision
is not formally represented \emph{within} the Bayesian inferential
apparatus itself, where the distinction between ``having beliefs''
and ``being committed to act'' belongs to the decision-theoretic
supplement. In risk-sensitive domains, the location of this boundary
is consequential: nothing inside the probabilistic formalism marks a
high posterior as insufficient for irreversible action.

In the clinical case, the posterior
$P(H_1) = 0.55$ completes the probabilistic inferential task;
whether it licenses action is a question the inferential formalism
does not itself pose. The $\kappa$--$\tau$ logic
instead represents the clinician's actual epistemic state:
both hypotheses are live, neither has crossed the commitment
threshold, and the system explicitly models \emph{justified
non-commitment} as an inferential outcome rather than as
incomplete reasoning.

\emph{(iv) Emergent gain above the join.}
Under constructive interaction ($\kappa^*>0$), the synthesis operator
produces composite plausibility that exceeds the plausibility
of the strongest individual component---strictly, whenever the join
is not already saturated (this is established formally, with its
exact strictness conditions, in
Theorem~\ref{thm:absorption-dominance}). Probabilistic composition
of events exhibits no analogue: joint occurrence satisfies
$P(A \cap B) \le \min(P(A), P(B))$, and the additive carry-over in
$P(A \cup B)$ is fixed extensionally by overlap, insensitive to any
interaction relation between the hypotheses as explanations. This
gain above the join
is the formal signature of emergent explanation---the ``whole
greater than the parts'' phenomenon that the framework is designed
to capture---and, crucially, it arises only when the components
interact constructively, have nonzero scores, and the join is
unsaturated (Theorem~\ref{thm:absorption-dominance}). The mere accumulation of non-interacting
hypotheses produces no gain above the join, which is precisely the
property a governance logic for risk-sensitive domains requires:
commitment cannot be manufactured by composing weak, unrelated
hypotheses.

\paragraph{Scope of the risk claim.}
A deflationary clarification is in order before proceeding to the
formal development. The $\kappa$--$\tau$ logic is
\emph{parametrically} risk-sensitive: it internalizes the
\emph{locus} at which a domain's risk posture enters
inference---through the threshold $\tau$, the activation floor
$\epsilon$, and the margin function $\delta$ of the governance
policy---but it does not itself contain a theory of risk. The
formalism represents no actions, losses, reversibility measures,
delay costs, or catastrophic-outcome probabilities, and it derives
no threshold values from such quantities. When, in what follows,
crossing $\tau$ is glossed as the point at which acting has become
preferable to waiting, this is the \emph{intended interpretation} of
an externally supplied parameter, not a theorem of the system: the
logic certifies that the governance threshold has been crossed, and
the burden of calibrating that threshold to the stakes---the mapping
from a risk profile $r$ to governance parameters
$\langle \tau_r, \delta_r \rangle$---lies with the deployment
context, a calibration problem registered explicitly among the
future directions (Section~\ref{sec:future-directions}). The stronger
decision-theoretic claim would require augmenting the formalism with
exactly the loss-theoretic apparatus whose absence from graded
inference the present section has been describing.

\paragraph{Summary.}

Probability models \emph{how much one believes each hypothesis}.
Plausibility models \emph{how explanatorily active each hypothesis is}.
The two frameworks are closest when hypotheses do not interact and
commitment is immediate---but even there they do not coincide:
plausibility weights are not normalized across a partition,
connectives are lattice-valued, and negation is complementation, so
the $\kappa$-free, instant-commitment fragment is a graded valuation
semantics, not a probability calculus. The distinctive divergence,
however, appears when hypotheses interact---when the explanatory
work done by one hypothesis changes
what can be learned from another---and when the timing of
commitment is itself a governed decision. The narrower and stronger
form of the contrast is this: the $\kappa$--$\tau$ semantics does not
require normalized weights, and it internalizes an explicit
commitment judgment over interaction-sensitive explanatory
structures. These are the conditions
under which the $\kappa$--$\tau$ logic operates.

\section{Scoring Semantics}
\label{sec:scoring-semantics}

The scoring semantics translates the conceptual distinctions developed
above into a concrete valuation mechanism. Each content formula
receives a score in $[0,1]$ reflecting its current epistemic standing
within an abductive state; judgments, as fixed by the typing of
Section~\ref{sec:language}, are not score-bearing and are evaluated
by the satisfaction conditions below. For atomic hypotheses, scores
coincide with the assigned plausibility weights; for compound content
formulas, scores are computed compositionally. The critical departure
from standard many-valued semantics lies in the synthesis operator
$\otimes$, whose valuation depends not only on the scores of its
operands but also on their epistemic interaction as encoded by
$\kappa^*$. The propositional connectives receive the usual fuzzy
lattice interpretation (min, max, complementation)---the valuation is
De Morgan-style, not Boolean: excluded middle and non-contradiction
do not in general receive classical values---while the judgment
operators $P$ and $C_\tau$ impose threshold conditions that
distinguish epistemic activation from normative commitment.

We define a score function $\mathrm{sc}_S(\cdot)$.

\subsection{Base Valuation}

For atomic hypotheses:

\[
\mathrm{sc}_S(h) = w(h)
\]

For atomic observations:

\[
\mathrm{sc}_S(o) = 1
\qquad (o \in \mathsf{O})
\]

For Boolean connectives:

\[
\mathrm{sc}_S(\varphi \vee \psi) =
\max(\mathrm{sc}_S(\varphi), \mathrm{sc}_S(\psi))
\]

\[
\mathrm{sc}_S(\varphi \wedge \psi) =
\min(\mathrm{sc}_S(\varphi), \mathrm{sc}_S(\psi))
\]

\[
\mathrm{sc}_S(\neg \varphi) =
1 - \mathrm{sc}_S(\varphi)
\]

In the definitions above, atomic observations are assigned the maximal valuation, reflecting their role in representing accepted informational inputs rather than competing explanatory constructs.

\paragraph{On the choice of negation.}
Negation is involutive: $\mathrm{sc}_S(\neg\varphi)=1-\mathrm{sc}_S(\varphi)$.
This is a deliberate design choice and the one point at which the
framework imposes a duality reminiscent of additive valuations. Its
scope is limited: the framework's non-additivity resides in the
\emph{synthesis} operator $\otimes$, which is where competing
explanations actually interact, and not in the treatment of a formula
against its own negation. One consequence is explicit: a hypothesis and
its negation cannot both be plausible once the activation floor exceeds
one half, since $P\varphi$ and $P\neg\varphi$ require
$\mathrm{sc}_S(\varphi)\ge\epsilon$ and $1-\mathrm{sc}_S(\varphi)\ge\epsilon$,
jointly satisfiable only for $\epsilon\le\tfrac12$. The coexistence the
framework is designed to model is coexistence of \emph{distinct}
hypotheses $H_1,H_2$, each with its own weight, not of a hypothesis and
its complement; the latter is correctly excluded above
$\epsilon=\tfrac12$. A paired necessity/possibility valuation that
lifts this restriction is discussed as a refinement in
Section~\ref{sec:future-directions}.
\paragraph{Range truncation.}
To preserve $[0,1]$-valuedness under interaction, we use the standard
clamping operator:
\[
[x]_0^1 := \min(1,\max(0,x)).
\]
\begin{lemma}[Range preservation]
For every content formula $\chi$, $\mathrm{sc}_S(\chi)\in[0,1]$.
\end{lemma}

\begin{proof}
By structural induction on $\chi$. The base cases are immediate.
The propositional clauses preserve $[0,1]$, and the $\otimes$ clause is
$[0,1]$-valued by definition of $[\cdot]_0^1$. Judgments carry no
score, so no further cases arise. \qed
\end{proof}

\subsection{Lifting $\kappa$ to Hypothesis Terms}

Interaction $\kappa$ is elicited at the level of atomic hypotheses:

\[
\kappa : \mathsf{H} \times \mathsf{H} \rightarrow [-1,1]
\]

Since $\otimes$ may generate compound hypothesis terms,
we define lifting via constituent aggregation.

Let $\mathrm{At}(t)$ denote the set of atomic hypotheses
occurring in term $t$:

\[
\mathrm{At}(h)=\{h\}, \qquad
\mathrm{At}(t_1 \otimes t_2)=
\mathrm{At}(t_1)\cup\mathrm{At}(t_2)
\]

We define the lifted interaction $\kappa^*$:

\[
\kappa^*(t_1,t_2)=
\frac{1}{|\mathrm{At}(t_1)||\mathrm{At}(t_2)|}
\sum_{h_i \in \mathrm{At}(t_1)}
\sum_{h_j \in \mathrm{At}(t_2)}
\kappa(h_i,h_j)
\]

The lifted interaction $\kappa^*$, therefore, represents the average
epistemic compatibility between the constituent explanatory
structures underlying the compound terms. In accordance with the
structural-equivalence convention of Section~\ref{sec:language},
$\kappa^*(t_1,t_2)=0$ whenever $t_1 \equiv_\otimes t_2$:
repetition of the same explanatory content, in either order, carries
no interaction. Terms that are not $\equiv_\otimes$-equivalent---%
including differently bracketed terms over the same atoms, which
Section~\ref{sec:algebraic-semantics} shows to be genuinely distinct
explanatory structures---interact through the averaging formula.

\paragraph{Extension to positive content formulas.}
Hypothesis terms are the primary domain of $\otimes$, but the
structural property of Section~\ref{sec:structural-properties} requires
synthesis to interact with disjunction, and the grammar of
Section~\ref{sec:language} accordingly admits $\otimes$ on all
positive content formulas. We therefore extend the atom map
to content formulas by
\[
\mathrm{At}(\neg\chi)=\mathrm{At}(\chi),\qquad
\mathrm{At}(\chi\vee\psi)=\mathrm{At}(\chi\wedge\psi)
=\mathrm{At}(\chi)\cup\mathrm{At}(\psi),
\]
and define $\kappa^*$ on any pair of content formulas through their
atom sets exactly as above, subject to the structural-equivalence
convention (extended to content formulas as the smallest congruence
generated by commutativity of $\otimes$, $\vee$, and $\wedge$). The synthesis clause then applies to any positive
$\chi,\psi$, since it depends only on their scores and on
$\kappa^*(\chi,\psi)$. That $\otimes$ is not applied to judgments,
nor to formulas whose principal connective is negation, is enforced
by the grammar itself: synthesis composes
explanatory content, and the complement of a hypothesis is not itself
an explanatory structure. (The atom map is nonetheless defined through
negation so that atom sets exist for all content formulas.)

\paragraph{Illustrative examples.}

Consider atomic interactions among hypotheses:

\[
\kappa(h_1,h_3)=0.8, \quad
\kappa(h_1,h_4)=0.6, \quad
\kappa(h_2,h_3)=0.4, \quad
\kappa(h_2,h_4)=-0.2
\]

Let:

\[
t_1 = h_1 \otimes h_2, \qquad
t_2 = h_3 \otimes h_4
\]

Then:

\[
\mathrm{At}(t_1)=\{h_1,h_2\}, \quad
\mathrm{At}(t_2)=\{h_3,h_4\}
\]

and

\[
\kappa^*(t_1,t_2)
=
\frac{1}{4}
(0.8 + 0.6 + 0.4 - 0.2)
=
0.4
\]

Despite one negative interaction, the aggregate compatibility
remains positive.

\vspace{0.4em}

Now consider instead:

\[
\kappa(h_1,h_3)=0.2, \quad
\kappa(h_1,h_4)=-0.7, \quad
\kappa(h_2,h_3)=0.1, \quad
\kappa(h_2,h_4)=-0.6
\]

Then:

\[
\kappa^*(t_1,t_2)
=
\frac{1}{4}
(0.2 - 0.7 + 0.1 - 0.6)
=
-0.25
\]

Here, limited positive compatibilities are outweighed by
systematic tension.

\subsection{Synthesis Operator}

For hypothesis terms:

\[
\mathrm{sc}_S(t_1 \otimes t_2) =
\Big[
\max\!\big(\mathrm{sc}_S(t_1),\,\mathrm{sc}_S(t_2)\big)
+
\lambda\, \kappa^*(t_1,t_2)\,
\mathrm{sc}_S(t_1)\mathrm{sc}_S(t_2)
\Big]_0^1
\]

\noindent where $\lambda \in (0,1]$ is a fixed interaction-scaling parameter
controlling the influence of epistemic interaction relative to the baseline
aggregation. The baseline is the lattice join: at zero interaction, synthesis
returns the more plausible component (Theorem~\ref{thm:kappa-free-reduction}),
and emergent gain above it is produced \emph{solely} by constructive
interaction, while destructive interaction attenuates below the join.

More expressive lifting operators are possible and left for future work.

The synthesis operator can be interpreted as a logical abstraction of
the mixing dynamics underlying the $\kappa$--$\tau$ framework,
where explanatory states may interact constructively or destructively
prior to commitment.

\subsection{Satisfaction}

Relative to the activation floor $\epsilon$ carried by $S$:

\[
S \Vdash P\varphi \iff \mathrm{sc}_S(\varphi) \ge \epsilon
\]

\[
S \Vdash C_{\tau}\varphi \iff \mathrm{sc}_S(\varphi) \ge \tau
\]

Commitment is normative rather than epistemic necessity.

\section{Inference Principles}
\label{sec:inference}

\paragraph{No Forced Collapse}

\[
P\varphi \nRightarrow C_{\tau}\varphi
\]

Plausibility does not imply commitment.

\paragraph{Threshold Monotonicity}

If $\tau_2 > \tau_1$:

\[
C_{\tau_2}\varphi \Rightarrow C_{\tau_1}\varphi
\]

\paragraph{Constructive Synthesis}

If $P t_1$, $P t_2$, $\kappa^*(t_1,t_2) \ge 0$:

\[
P(t_1 \otimes t_2)
\]

Negative interactions discourage but do not prevent synthesis.

\begin{proposition}[Constructive synthesis preserves plausibility]
\label{prop:constructive-synthesis}
Let $S$ be an abductive state with activation floor $\epsilon$ and
let $t_1,t_2$ be hypothesis terms. If
$S \Vdash P t_1$, $S \Vdash P t_2$, and $\kappa^*(t_1,t_2)\ge 0$, then
$S \Vdash P(t_1 \otimes t_2)$.
\end{proposition}

\begin{proof}
From $S \Vdash P t_i$ we have $\mathrm{sc}_S(t_i)\ge \epsilon$ for $i=1,2$.
Let $a=\mathrm{sc}_S(t_1)$ and $b=\mathrm{sc}_S(t_2)$. By the synthesis clause,
with $\lambda\ge 0$ and $\kappa^*(t_1,t_2)\ge 0$,
\[
\mathrm{sc}_S(t_1\otimes t_2)
=
\big[\max(a,b)+\lambda\,\kappa^*(t_1,t_2)\,ab\big]_0^1
\ge \max(a,b) \ge \epsilon,
\]
since $\lambda\,\kappa^*(t_1,t_2)\,ab\ge 0$ and clamping cannot decrease a
value already in $[0,1]$. Thus $S \Vdash P(t_1\otimes t_2)$.
\qed
\end{proof}

\section{Basic Structural Property}
\label{sec:structural-properties}

\paragraph{Non-Distributivity (Typical Case)}

In general:

\[
\varphi \otimes (\psi \vee \chi)
\not\equiv
(\varphi \otimes \psi) \vee (\varphi \otimes \chi)
\]

\textit{Intuition.}
Epistemic interaction introduces contextual dependencies
absent in standard distributive logics.

\paragraph{Counterexample with interaction.}

Treat $\varphi,\psi,\chi$ as atomic hypotheses with
$\mathrm{sc}_S(\varphi)=0.6$,
$\mathrm{sc}_S(\psi)=0.5$,
$\mathrm{sc}_S(\chi)=0.4$,
$\lambda=1$, and atomic interactions
$\kappa(\varphi,\psi)=0.5$, $\kappa(\varphi,\chi)=-0.5$.
By the Boolean extension of $\kappa^*$,
$\kappa^*(\varphi,\psi\vee\chi)
=\tfrac12\big(\kappa(\varphi,\psi)+\kappa(\varphi,\chi)\big)
=\tfrac12(0.5-0.5)=0$, and
$\mathrm{sc}_S(\psi\vee\chi)=\max(0.5,0.4)=0.5$.

Then, on the left-hand side,
\[
\mathrm{sc}_S(\varphi \otimes (\psi\vee\chi))
= \max(0.6,0.5) + (0)(0.6)(0.5) = 0.6 .
\]
On the right-hand side,
\[
\mathrm{sc}_S(\varphi\otimes\psi)
= \max(0.6,0.5) + (0.5)(0.6)(0.5) = 0.75,
\]
\[
\mathrm{sc}_S(\varphi\otimes\chi)
= \max(0.6,0.4) + (-0.5)(0.6)(0.4) = 0.48,
\]
\[
\max(0.75,0.48)=0.75 .
\]

Thus:

\[
\mathrm{sc}_S(\varphi \otimes (\psi \vee \chi)) = 0.6
\;<\; 0.75 =
\mathrm{sc}_S((\varphi \otimes \psi)\vee(\varphi \otimes \chi)),
\]

the interaction $\kappa^*$ being applied throughout rather than silently
set to zero.

\section{Derivation and Closure}
\label{sec:derivation}

Standard logical systems---whether classical, intuitionistic, linear,
or modal---are oriented towards \emph{closure}: the semantic target of
derivation is a settled verdict of derivability relative to a theory,
and the belief state implicitly modelled by a derivational apparatus
is a closed one. This is a claim about the intended output of
inference, not about procedure: classical consequence is
undecidable, and proof search may of course fail to terminate. What
these systems do not represent is a \emph{stable but non-committed}
inferential state---a configuration in which derivation has done its
work, the valuation has stabilized, and yet closure is deliberately
withheld.

The $\kappa$--$\tau$ logic introduces a structural distinction between:

\begin{itemize}
    \item \textbf{Derivability}
    \item \textbf{Commitment / Closure}
\end{itemize}

\subsection{Two Modes of Derivation}

Given an abductive state $S$, we distinguish:

\paragraph{Collapse Derivation}

\[
\vdash_S^c \varphi
\quad \text{iff} \quad
\mathrm{sc}_S(\varphi) \ge \tau
\]

Collapse derivations represent commitment-worthy conclusions.
They behave analogously to standard logical derivations:

\begin{itemize}
    \item closing: they issue a settled, commitment-eligible verdict
    \item stable under threshold monotonicity
    \item normatively endorsed
\end{itemize}

\paragraph{Suspended Derivation}

\[
\vdash_S^p \varphi
\quad \text{iff} \quad
\epsilon \le \mathrm{sc}_S(\varphi) < \tau
\]

Suspended derivations capture epistemically justified but
non-committed conclusions.

They represent inferentially valid states in which:

\begin{itemize}
    \item derivation is active but non-closing
    \item hypotheses remain epistemically live
    \item closure is governance-regulated
\end{itemize}

The two notions relate by
\[
\vdash_S^p \varphi
\iff
S\Vdash P\varphi \ \text{and}\ S\not\Vdash C_\tau\varphi,
\qquad\text{equivalently}\qquad
\epsilon \le \mathrm{sc}_S(\varphi) < \tau .
\]
Thus $P\varphi$ records activation regardless of commitment (it may hold
even when $\mathrm{sc}_S(\varphi)\ge\tau$), whereas $\vdash_S^p\varphi$
additionally requires non-commitment. A formula with
$\mathrm{sc}_S(\varphi)\ge\tau$ satisfies both $P\varphi$ and
$C_\tau\varphi$ and is collapse-derivable, $\vdash_S^c\varphi$, not
suspended.

\subsection{Derivation vs Commitment}

Unlike standard logical frameworks, the framework does not identify
derivability with closure.

\[
\vdash_S^p \varphi \;\nRightarrow\; \vdash_S^c \varphi
\]

Thus, plausibility does not imply commitment.

This distinction holds independently of the underlying logical
tradition. In standard systems:

\begin{itemize}
    \item classical logic enforces bivalence
    \item intuitionistic logic enforces proof conditions
    \item linear logic enforces resource constraints
\end{itemize}

but the intended output of derivation remains a closed verdict.

The framework instead models:

\begin{center}
\textbf{Non-closing inference---stable suspension---as a first-class
epistemic condition}
\end{center}

\subsection{Validity and Consequence}
\label{subsec:consequence}

The relations $\vdash_S^p$ and $\vdash_S^c$ are defined per state by the
satisfaction conditions of Section~\ref{sec:scoring-semantics}. They
induce a consequence relation once we quantify over a class of states.
Fix a class $\mathcal{C}_{\epsilon,\tau}$ of admissible
abductive states over models sharing a common activation floor
$\epsilon$ \emph{and} a common commitment threshold $\tau$ (for
instance, all states over models with a fixed $\mathsf{H}$ respecting
the diagonal convention $\kappa(h,h)=0$ and carrying the same
governance parameters). For a set
$\Gamma\cup\{\chi\}$ of \emph{content} formulas, define
\emph{plausibility consequence}
\[
\Gamma \models_{\epsilon} \chi
\quad\text{iff}\quad
\forall S\in\mathcal{C}_{\epsilon,\tau}\colon\
\Big(\forall\gamma\in\Gamma,\ \mathrm{sc}_S(\gamma)\ge\epsilon\Big)
\ \Rightarrow\ \mathrm{sc}_S(\chi)\ge\epsilon,
\]
and \emph{commitment consequence} $\models_{\tau}$ by replacing $\epsilon$
with $\tau$ throughout---well defined precisely because the class
fixes both thresholds: were $\tau$ allowed to vary across the class,
$\models_{\tau}$ would have no determinate meaning. Validity is the
nullary case:
$\models_{\epsilon}\chi$ iff $\mathrm{sc}_S(\chi)\ge\epsilon$ for
every $S\in\mathcal{C}_{\epsilon,\tau}$. Consequence is defined over content formulas
only, in keeping with the typing of Section~\ref{sec:language}:
judgments are the form in which threshold satisfaction is asserted,
not premises or conclusions of the graded consequence relations.
The per-state turnstiles $\vdash_S^p,\vdash_S^c$
are the instances of these relations at a single state; $\models$
quantifies over the class.

This paper supplies the \emph{semantics} of the $\kappa$--$\tau$ logic and
the consequence relation it induces. A Hilbert- or sequent-style
axiomatisation, together with soundness and completeness for the
$\kappa$-free fragment---which, reducing to a bounded lattice with an
internal threshold projection, is expected to be tractable---is developed
in companion work. We flag this division explicitly rather than leaving
the proof-theoretic status implicit: the present contribution is a
semantic kernel and its consequence relation, not a proof calculus.

\subsection{Dynamics of Suspended Derivation}

An evidence stream induces a sequence of states
\[
S_0 \rightarrow S_1 \rightarrow S_2 \rightarrow \dots
\]
that differ only in their weight function. Each update consumes an
observation bundle $E_{t+1}\subseteq\mathsf{O}$. We equip the state with
an \emph{evidence-compatibility map}
\[
\kappa_o : \mathsf{O}\times\mathsf{H} \to [-1,1],
\]
where $\kappa_o(o,h)$ records how far observation $o$ supports ($>0$) or
undercuts ($<0$) hypothesis $h$, together with a reliability weighting
$\sigma:\mathsf{O}\to[0,1]$. This is the channel through which the
maximal valuation $\mathrm{sc}_S(o)=1$ enters the dynamics: an accepted
observation contributes its full reliability-weighted compatibility.

\paragraph{Update rule.}
For each atomic hypothesis $h$ and learning rate $\eta\in(0,1]$,
\[
w_{t+1}(h) \;=\;
\Big[\, w_t(h) + \eta \sum_{o\in E_{t+1}} \sigma(o)\,\kappa_o(o,h) \,\Big]_0^1 .
\]
Scores at $t+1$ are recomputed from $w_{t+1}$ by the (unchanged) scoring
semantics; $\kappa$, $\tau$, and $\lambda$ are held fixed across the
stream.

\paragraph{Reinforcement.}
The bundle $E_{t+1}$ \emph{reinforces} $\varphi$, with hypothesis term
$t_\varphi$, if its aggregate support is nonnegative on every constituent:
\[
\Delta_{t+1}(h) :=
\sum_{o\in E_{t+1}} \sigma(o)\,\kappa_o(o,h) \;\ge\; 0
\qquad \text{for all } h\in\mathrm{At}(t_\varphi).
\]

\begin{definition}[Hereditarily constructive terms]
\label{def:hereditarily-constructive}
A hypothesis term is \emph{hereditarily constructive} in $S$ if it is
atomic, or of the form $t_1 \otimes t_2$ with $t_1,t_2$ hereditarily
constructive and $\kappa^*(t_1,t_2) \ge 0$: every interaction node in
the term, not merely the outermost one, is non-negative.
\end{definition}

\begin{proposition}[Monotonicity under reinforcement]
\label{prop:monotone-update}
Let $\varphi$ have term $t_\varphi$ and suppose $E_{t+1}$ reinforces
$\varphi$. If $t_\varphi$ is hereditarily constructive, then
$\mathrm{sc}_{S_{t+1}}(\varphi)\ge\mathrm{sc}_{S_t}(\varphi)$.
\end{proposition}

\begin{proof}
We first show, by structural induction, that the score of a
hereditarily constructive term is monotone non-decreasing in every
atomic weight. For atomic $t$ this is immediate. For
$t = t_1 \otimes t_2$ with $\kappa^*(t_1,t_2)\ge 0$, the synthesis
value $[\max(a,b)+\lambda\kappa^*ab]_0^1$ is monotone non-decreasing
in each of $a,b$ (both $\max$ and the product term are, and clamping
preserves monotonicity); by the induction hypothesis
$a=\mathrm{sc}_S(t_1)$ and $b=\mathrm{sc}_S(t_2)$ are themselves
monotone in the atomic weights, and the composition of monotone maps
is monotone. Reinforcement gives $\Delta_{t+1}(h)\ge 0$, hence
$w_{t+1}(h)\ge w_t(h)$ for every $h\in\mathrm{At}(t_\varphi)$, and
the claim follows.
\qed
\end{proof}

\paragraph{Remark (necessity of the restriction).}
Hereditary constructiveness cannot be weakened to non-negativity of
the outermost interaction alone. If an operand is itself a
destructively interacting composite, reinforcing one of its
constituents can \emph{lower} the operand's score---take
$t_1 = h_1 \otimes h_2$ with $\kappa(h_1,h_2)<0$: raising $w(h_1)$
while $h_2$ remains the stronger component increases the magnitude of
the destructive term $\lambda\kappa^*w(h_1)w(h_2)$ without moving the
join, so $\mathrm{sc}(t_1)$ strictly decreases---and this decrease
propagates through an outer synthesis even when the outer interaction
is zero or positive. The failure is not a defect but the intended
semantics of destructive interaction, recorded in the following
remark; the proposition simply delimits the configurations in which
reinforcement is guaranteed to help.

\paragraph{Remark (destructive case).}
If a term contains a node with $\kappa^*<0$, increasing the weight of
one mutually-inhibiting
component can \emph{lower} the relevant synthesis score. This is
intended: it
records that strengthening one of two antagonistic explanations weakens
their joint plausibility. Monotonicity is therefore a property of
reinforcement under hereditarily constructive configuration, not an
unconditional law.

\paragraph{Stabilisation.}
Starting from $\vdash_{S_t}^p\varphi$, reinforcing updates raise
$\mathrm{sc}(\varphi)$ monotonically until one of the following occurs:

\begin{enumerate}
    \item \textbf{Commitment:} $\mathrm{sc}_{S_{t'}}(\varphi)\ge\tau$,
    triggering $\vdash_{S_{t'}}^c\varphi$ (collapse).
    \item \textbf{Disconfirmation:} a later bundle yields $\Delta(h)<0$ for
    some $h\in\mathrm{At}(t_\varphi)$, ending reinforcement.
    \item \textbf{Pre-emption:} a competing term achieves commitment
    first under the governance policy $\Pi$ in force
    (Section~\ref{subsec:governed-competition}), superseding
    continued suspension of its competitors.
    \item \textbf{Fixed point:} the stream is exhausted, or
    $\mathrm{sc}_{S_{t'+1}}(\varphi)=\mathrm{sc}_{S_{t'}}(\varphi)$, leaving
    $\varphi$ in stable suspended derivation.
\end{enumerate}

\paragraph{Interpretation}

Suspended derivations therefore behave as:

\begin{itemize}
    \item inferential fixed points
    \item epistemically stable states
    \item analytically meaningful outputs
\end{itemize}

rather than incomplete reasoning.

\subsection{Governed Competition}
\label{subsec:governed-competition}

Collapse derivability marks a conclusion as commit-worthy:
$\vdash_S^c \varphi$ records that the evidence supporting $\varphi$
has crossed the threshold $\tau$. Commit-worthiness, however, is a
property of a conclusion in isolation, and risk-sensitive commitment
cannot be adjudicated in isolation: a conclusion that barely clears
the threshold while a strongly incompatible rival stands just behind
it has not resolved the rivalry, only outpaced it. The governance
policy $\Pi$ carried by the model therefore regulates the
\emph{commitment event}---the collapse itself---over and above
threshold satisfaction.

\begin{definition}[Rival-sensitive commitment]
\label{def:rival-sensitive}
Let $\delta : (0,1] \times (0,1] \to [0,1)$ be a \emph{margin
function}, non-decreasing in each argument, with
$\delta(\tau,x) \to 0$ as $x \to 0$. A formula $\varphi$ with term
$t_\varphi$ is \emph{committed} in $S$, written
$\mathrm{Commit}_S(\varphi)$, iff
\begin{enumerate}
\item $\vdash_S^c \varphi$, and
\item for every active rival $\psi$---that is, every $\psi$ with term
$t_\psi$ such that $\mathrm{sc}_S(\psi) \ge \epsilon$ and
$\kappa^*(t_\varphi,t_\psi) < 0$---the separation condition holds:
\[
\mathrm{sc}_S(\varphi) - \mathrm{sc}_S(\psi)
\;\ge\;
\delta\big(\tau,\, -\kappa^*(t_\varphi,t_\psi)\big).
\]
\end{enumerate}
\end{definition}

The required margin thus grows with both the stakes, represented by
$\tau$, and the strength of the incompatibility, represented by
$-\kappa^*$: the more severe the rivalry, the wider the separation
demanded before collapse. Compatible or mildly interacting
alternatives ($\kappa^* \ge 0$) impose no margin---coexisting
constructive explanations are candidates for synthesis, not rivals
to be outrun.

The \emph{first-past-$\tau$} rule---commit to the first conclusion
whose score reaches $\tau$---is recovered as the degenerate policy
$\delta \equiv 0$. It is appropriate only in low-rivalry regimes,
and its inadequacy in risk-sensitive ones is easy to exhibit:
with two strongly incompatible conclusions scoring just above and
just below the threshold---separated by less than the noise of any
realistic estimation procedure---first-past-$\tau$ commits to the
leader, and an infinitesimal perturbation or a different
evidence-arrival order would have committed to its rival. Committing
under such conditions is precisely the premature convergence that
governed abduction exists to prevent; under any non-degenerate
margin function, the state instead remains one of suspended
derivation until the rivalry is genuinely resolved---by the leader
pulling away, or by the rival dropping below activation.

Two structural observations follow. First, rival-sensitive
commitment strengthens, and never weakens, the threshold condition:
$\mathrm{Commit}_S(\varphi)$ implies $\vdash_S^c \varphi$, so all
results stated for collapse derivability apply \emph{a fortiori} to
committed conclusions, and the threshold projection $\Theta_\tau$ of
Section~\ref{sec:algebraic-semantics} retains its role as the
semantic counterpart of commit-worthiness, with $\Pi$ acting as a
filter on top of it. Second, the policy preserves No Forced
Collapse in a stronger form: not only does plausibility not imply
commitment, but even commit-worthiness does not---the transition to
action requires clearance of the rivalry structure, so that
$\kappa$, through the margin condition, directly co-governs the
timing of collapse rather than merely shaping scores.

\subsection{Key Structural Consequence}

The framework supports derivational states that stabilize without closing.

\begin{center}
\textbf{The pre-collapse state is a valid inferential outcome}
\end{center}

Closure becomes a normative decision event rather than a
logical necessity.
\section{Algebraic Semantics}
\label{sec:algebraic-semantics}

The scoring semantics introduced earlier admits a natural algebraic
interpretation in which inference is treated as a structured combination
and transformation of graded epistemic valuations rather than as
truth assignment over fixed models.

This perspective aligns the present semantics with many-valued and algebraic
logical traditions rather than with Tarskian satisfaction-based
semantics. Unlike truth-conditional systems, the framework operates over evolving,
graded valuations whose inferential role is governed by interaction
and normative thresholds.

\subsection{Underlying Algebraic Structure}

Let $\mathcal{A} = \langle [0,1], \wedge, \vee, \oplus, \otimes_P \rangle$
be a graded valuation algebra where:

\begin{itemize}
    \item $\wedge = \min$ (lattice meet),
    \item $\vee = \max$ (lattice join),
    \item $\otimes_P(a,b) = ab$ (product operator),
    \item $\oplus(a,b) = a + b - ab$ (probabilistic sum).
\end{itemize}

The structure is to be characterized with care. The reduct
$\langle [0,1], \min, \max, 1-(\cdot) \rangle$ is a De Morgan
(indeed Kleene) algebra; $\otimes_P$ is the product t-norm and
$\oplus$ its dual t-conorm, both extensively studied in
many-valued and fuzzy logical frameworks
\cite{hajek1998metamathematics,klement2000triangular}. The ensemble
is thus a bounded distributive lattice expanded with a
t-norm/t-conorm pair---\emph{not} a product t-norm algebra in the
residuated sense of \cite{hajek1998metamathematics}: logical
conjunction is interpreted by $\min$, not by the product, and no
residuum or implication operation is assumed, none being needed for
the valuation semantics. The distinctive feature of
the present semantics lies not in the base algebraic structure, but in
the interaction-modulated composition induced by $\kappa^*$: synthesis
takes the lattice join $\max$ as its baseline and perturbs it by a
$\kappa^*$-scaled product term.

This places the synthesis operator in the lineage of \emph{compensatory}
aggregation operators, notably Zimmermann and Zysno's $\gamma$-operators
\cite{zimmermann1980latent}, which interpolate between a t-norm
(intersection, no compensation) and a t-conorm (union, full compensation)
through a single compensation parameter. What is new here is twofold: the
compensation coefficient is not a global $\gamma$ but a hypothesis-pair-indexed
epistemic relation $\kappa^*$, and the commitment threshold $\tau$ is
internal to the logic rather than an external decision rule.

\subsection{Scoring as Algebraic Valuation}

Within this framework, the scoring function:

\[
\mathrm{sc}_S : \varphi \rightarrow [0,1]
\]

acts as an algebraic valuation.

Disjunction and conjunction correspond to lattice operations:

\[
\mathrm{sc}_S(\varphi \vee \psi)
=
\max(\mathrm{sc}_S(\varphi), \mathrm{sc}_S(\psi))
\]

\[
\mathrm{sc}_S(\varphi \wedge \psi)
=
\min(\mathrm{sc}_S(\varphi), \mathrm{sc}_S(\psi))
\]

Negation is defined by standard complementation:

\[
\mathrm{sc}_S(\neg \varphi)
=
1 - \mathrm{sc}_S(\varphi)
\]

\subsection{Interaction as Algebraic Operator}

Epistemic interaction $\kappa$ enriches the lattice-join baseline
$\max$ by a $\kappa^*$-scaled product term:

\[
\mathrm{sc}_S(t_1 \otimes t_2)
=
\max\!\big(\mathrm{sc}_S(t_1),\mathrm{sc}_S(t_2)\big)
+
\lambda \kappa^*(t_1,t_2)\,
\mathrm{sc}_S(t_1)\mathrm{sc}_S(t_2)
\]

Thus, synthesis is the lattice join algebraically perturbed by interaction.

Constructive interaction increases composite valuation above the join,
while destructive interaction attenuates it below the join.

\begin{theorem}[Join reduction]
\label{thm:kappa-free-reduction}
If $\kappa^*(t_1,t_2)=0$, the synthesis operator coincides with the
lattice join on scores:
\[
\mathrm{sc}_S(t_1 \otimes t_2)
=
\max\!\big(\mathrm{sc}_S(t_1),\mathrm{sc}_S(t_2)\big)
=
\mathrm{sc}_S(t_1 \vee t_2).
\]
\end{theorem}

\begin{proof}
Immediate from the synthesis clause at $\kappa^*=0$; clamping is inert
since $\max(a,b)\in[0,1]$.
\qed
\end{proof}

\noindent\emph{Remark.} Synthesis and disjunction agree numerically only
in the non-interacting case. They remain conceptually distinct: $\vee$ is
defined on all formulas and is interaction-blind, whereas $\otimes$ is the
interaction-bearing operator on hypothesis terms. Emergence is precisely
the gap
\[
\Delta_\otimes
= \big[\max(a,b)+\lambda\kappa^*ab\big]_0^1 - \max(a,b).
\]
When clamping is inactive, $\Delta_\otimes = \lambda\kappa^*ab$; in
general the clamp intervenes, and the exact conditions for a strictly
positive gap---constructive interaction, nonzero component scores,
unsaturated join---are those of
Theorem~\ref{thm:absorption-dominance}.

\begin{theorem}[Monotonicity in interaction]
For fixed $a,b \in [0,1]$ and $\lambda \in (0,1]$, the synthesis valuation
is monotone non-decreasing in $\kappa$.
\end{theorem}

\begin{proof}
For the unclamped synthesis expression:
\[
\frac{\partial}{\partial \kappa}
\big(\max(a,b)+\lambda \kappa ab\big)
=
\lambda ab \ge 0,
\]
and clamping, being a monotone map, preserves monotonicity.
\qed
\end{proof}

Beyond monotonicity, interaction effects admit precise quantitative
characterization.

\begin{theorem}[Bounded perturbation]
\label{thm:bounded-perturbation}
For hypothesis terms $t_1,t_2$ with
$a=\mathrm{sc}_S(t_1)$ and $b=\mathrm{sc}_S(t_2)$,
the deviation from the lattice-join baseline is bounded by:
\[
\big|
\mathrm{sc}_S(t_1 \otimes t_2)
-
\max(a,b)
\big|
\le
\lambda |\kappa^*(t_1,t_2)|\, ab.
\]
\end{theorem}

\begin{proof}
The unclamped deviation equals $\lambda \kappa^* ab$.
Taking absolute values yields the bound.
Clamping cannot increase deviation.
\qed
\end{proof}

Figure~\ref{fig:kappa-perturbation-slice} visualizes the $\kappa^*$-perturbation
as a deformation of the lattice-join baseline, and shows how the
normative threshold $\tau$ defines a collapse boundary.

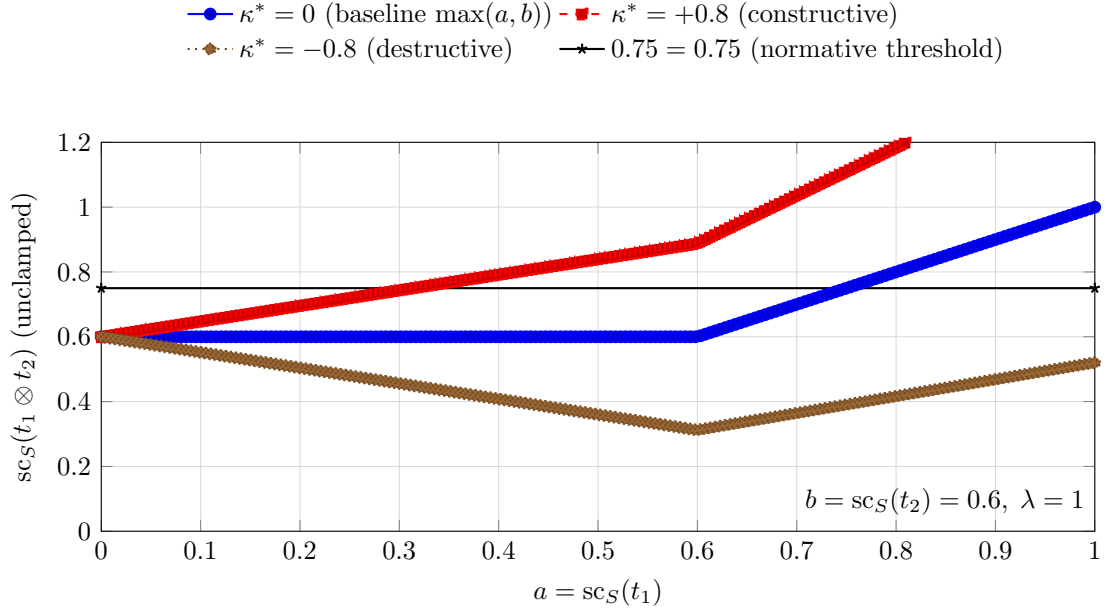
\begin{figure}[t]
\centering
\begin{tikzpicture}
\begin{axis}[
  width=0.92\linewidth,
  height=0.42\linewidth,
  xlabel={$a=\mathrm{sc}_S(t_1)$},
  ylabel={$\mathrm{sc}_S(t_1\otimes t_2)$ (unclamped)},
  xmin=0, xmax=1,
  ymin=0, ymax=1.2,
  domain=0:1,
  samples=200,
 legend cell align=left,
legend columns=2,
legend style={
  font=\small,
  at={(0.5,1.18)},
  anchor=south,
  draw=none,
  fill=none
},
  ticklabel style={font=\small},
  label style={font=\small},
  legend cell align={left},
  grid=both,
  grid style={line width=.2pt, draw=gray!20},
  major grid style={line width=.2pt, draw=gray!30},
]

\def\b{0.6}
\def\lam{1.0}

\addplot+[thick] {max(x,\b)};

\addplot+[thick, dashed] {max(x,\b) + \lam*(0.8)*x*\b};

\addplot+[thick, dotted] {max(x,\b) + \lam*(-0.8)*x*\b};

\def\tau{0.75}
\addplot+[thick] coordinates {(0,\tau) (1,\tau)};

\legend{
{$\kappa^*=0$ (baseline $\max(a,b)$)},
$\kappa^*=+0.8$ (constructive),
$\kappa^*=-0.8$ (destructive),
$\tau=0.75$ (normative threshold)
}

\node[font=\small, anchor=south east] at (axis cs:1,0.02)
{$b=\mathrm{sc}_S(t_2)=0.6,\;\lambda=1$};

\end{axis}
\end{tikzpicture}
\caption{Effect of interaction on synthesis (slice view).  For fixed $b$, the $\kappa^*$-term perturbs the lattice-join baseline $\max(a,b)$ by an amount proportional to $ab$.  The threshold $\tau$ induces a normative boundary for collapse.}
\label{fig:kappa-perturbation-slice}
\end{figure}

\begin{theorem}[Non-associativity in general]
\label{thm:non-associativity}
In the presence of a nonzero interaction, the synthesis operator of the $\kappa$--$\tau$ logic is not associative in general.
\end{theorem}

\begin{proof}
Treat $t_1,t_2,t_3$ as atomic with
\[
\mathrm{sc}_S(t_1)=0.7,\quad \mathrm{sc}_S(t_2)=0.5,\quad
\mathrm{sc}_S(t_3)=0.6, \quad \lambda=1,
\]
and atomic interactions
\[
\kappa(h_1,h_2)=0.8,\quad
\kappa(h_2,h_3)=-0.6,\quad
\kappa(h_1,h_3)=0.2.
\]
The lifted interaction for a composite is \emph{determined} by the lifting
of Section~\ref{sec:scoring-semantics}, not chosen freely.

\emph{Left bracketing.}
\[
\mathrm{sc}_S(t_1\otimes t_2)
= \max(0.7,0.5)+(0.8)(0.7)(0.5)
= 0.7+0.28 = 0.98,
\]
\[
\kappa^*(t_1\otimes t_2,\,t_3)
= \tfrac12\big(\kappa(h_1,h_3)+\kappa(h_2,h_3)\big)
= \tfrac12(0.2-0.6) = -0.2,
\]
\[
\mathrm{sc}_S((t_1\otimes t_2)\otimes t_3)
= \max(0.98,0.6)+(-0.2)(0.98)(0.6)
= 0.98-0.1176 = 0.8624.
\]

\emph{Right bracketing.}
\[
\mathrm{sc}_S(t_2\otimes t_3)
= \max(0.5,0.6)+(-0.6)(0.5)(0.6)
= 0.6-0.18 = 0.42,
\]
\[
\kappa^*(t_1,\,t_2\otimes t_3)
= \tfrac12\big(\kappa(h_1,h_2)+\kappa(h_1,h_3)\big)
= \tfrac12(0.8+0.2) = 0.5,
\]
\[
\mathrm{sc}_S(t_1\otimes (t_2\otimes t_3))
= \max(0.7,0.42)+(0.5)(0.7)(0.42)
= 0.7+0.147 = 0.847.
\]

Since $0.8624 \neq 0.847$, associativity fails.
\qed
\end{proof}

\begin{theorem}[Semantic commutativity]
\label{thm:commutativity}
Since $\kappa$ is symmetric (Section~\ref{sec:language}), synthesis
is semantically commutative:
\[
\mathrm{sc}_S(t_1\otimes t_2)=\mathrm{sc}_S(t_2\otimes t_1).
\]
\end{theorem}

\begin{proof}
Follows from symmetry of $\kappa^*$ and symmetry of the synthesis
expression in $a,b$.
\qed
\end{proof}

Thus, the framework's context sensitivity resides in compositional structure rather
than argument permutation, reflecting the dependence of synthesis on
epistemic configuration rather than mere ordering.

\begin{proposition}[Idempotence up to structural equivalence]
\label{prop:idempotence-condition}
With the structural-equivalence convention of
Section~\ref{sec:language}, synthesis is idempotent on all hypothesis
terms, $\mathrm{sc}_S(t\otimes t)=\mathrm{sc}_S(t)$, and more
generally, for any terms $t,u$ with $t \equiv_\otimes u$,
\[
\mathrm{sc}_S(t\otimes u)
=\max\big(\mathrm{sc}_S(t),\mathrm{sc}_S(u)\big)
=\mathrm{sc}_S(t),
\]
the last equality because $\equiv_\otimes$-equivalent terms are
equiscoring under symmetric $\kappa$
(Theorem~\ref{thm:commutativity}).
\end{proposition}

\begin{proof}
By the convention, $\kappa^*(t,u)=0$ whenever $t \equiv_\otimes u$,
so the synthesis clause reduces to the join. With $u=t$ this gives
$\mathrm{sc}_S(t\otimes t)=\max(a,a)=a$; for $t \equiv_\otimes u$,
commutativity invariance gives $\mathrm{sc}_S(u)=\mathrm{sc}_S(t)$.
\qed
\end{proof}

\noindent The granularity of the convention is not innocuous
bookkeeping, and neither of the two more obvious choices would
suffice in its place. Character-for-character syntactic identity is
too fine: terms such as $t = h_1 \otimes h_2$ and
$u = h_2 \otimes h_1$---distinct strings carrying the same
explanatory content, and equal in score under symmetric $\kappa$
(Theorem~\ref{thm:commutativity})---would count as distinct, so
$\kappa^*(t,u)$, computed from their common atom set, would in
general be nonzero: $t \otimes t$ would be score-idempotent while
$t \otimes u$ inflated or attenuated the score, although $t$ and $u$
express the same composite. Identity of atom sets,
$\mathrm{At}(t)=\mathrm{At}(u)$, is conversely too coarse: it would
identify $(h_1 \otimes h_2) \otimes h_3$ with
$h_1 \otimes (h_2 \otimes h_3)$, terms which
Theorem~\ref{thm:non-associativity} shows may differ in score
precisely because their compositional structures differ, and which
therefore cannot be treated as the same explanatory content merely
because they are built from the same atoms. The commutativity
quotient $\equiv_\otimes$ sits exactly between the two: it identifies
what the semantics itself cannot distinguish, and nothing more.
Because the convention nullifies interaction on
$\equiv_\otimes$-equivalent terms and the baseline is the idempotent
join, repetition cannot inflate a score: $\otimes$ behaves as
governed synthesis, not as a conorm that accumulates mass from
repetition.

\begin{theorem}[Interaction dominance]
\label{thm:absorption-dominance}
Let $a=\mathrm{sc}_S(t_1)$, $b=\mathrm{sc}_S(t_2)$, and
$\lambda\in(0,1]$ as carried by the model.
\begin{enumerate}
\item If $\kappa^*(t_1,t_2)\ge 0$ then
$\mathrm{sc}_S(t_1\otimes t_2)\ge\max(a,b)$, with strict inequality
iff $\kappa^*(t_1,t_2)>0$, $ab>0$, and $\max(a,b)<1$.
\item If $\kappa^*(t_1,t_2)<0$ then $\mathrm{sc}_S(t_1\otimes t_2)<\max(a,b)$
whenever $ab>0$.
\end{enumerate}
\end{theorem}

\begin{proof}
The unclamped synthesis value is $\max(a,b)+\lambda\kappa^*ab$.
For (1), the perturbation $\lambda\kappa^*ab$ is non-negative, and
the lower clamp is inert, so the clamped value is at least
$\max(a,b)$. Strictness requires the perturbation to be positive
($\kappa^*>0$ and $ab>0$) \emph{and} the upper clamp to leave room:
if $\max(a,b)=1$, the clamped value equals $1=\max(a,b)$ regardless
of the perturbation, so no strict gain occurs---the upper clamp is
\emph{not} inert in the constructive case whenever the untruncated
expression exceeds one. If $\max(a,b)<1$, the clamped value
$\min(1,\max(a,b)+\lambda\kappa^*ab)$ strictly exceeds $\max(a,b)$.
For (2), since $\lambda|\kappa^*|ab \le ab \le \min(a,b) \le \max(a,b)$,
the lower clamp is inert, and the value
$\max(a,b)-\lambda|\kappa^*|ab$ is strictly below $\max(a,b)$ exactly
when $ab>0$. Thus emergent gain above the join occurs if and only if
the components interact constructively and the join is not already
saturated.
\qed
\end{proof}

By contrast, negative interaction may attenuate even strong terms,
capturing destructive interference.

\begin{proposition}[Inhibition and collapse thresholds]
\label{prop:threshold-kappa}
Let $a=\mathrm{sc}_S(t_1)$, $b=\mathrm{sc}_S(t_2)$, $ab>0$.

\[
f(a,b,\kappa)=\max(a,b)+\lambda \kappa ab.
\]

Collapse or inhibition thresholds follow by solving inequalities:

\[
\kappa^* \ge \frac{\tau-\max(a,b)}{\lambda ab}
\]

\[
\kappa^* \le \frac{\epsilon-\max(a,b)}{\lambda ab}.
\]
\end{proposition}

\begin{proof}
Direct algebraic solution.
\qed
\end{proof}

\subsection{Normative Threshold Projection}

\begin{definition}[Normative threshold projection]
For $\tau\in(0,1]$, define $\Theta_{\tau}:[0,1]\to[0,1]$ by
\[
\Theta_{\tau}(x)=
\begin{cases}
x & \text{if } x<\tau,\\
1 & \text{if } x\ge \tau.
\end{cases}
\]
\end{definition}

The projection $\Theta_{\tau}$ models collapse as a normative
commitment event rather than as logical necessity.

\begin{theorem}[Properties of threshold projection]
$\Theta_{\tau}$ is monotone, idempotent, and antitone in $\tau$.
\end{theorem}

\begin{proof}
Immediate by case inspection. \qed
\end{proof}

\paragraph{Remark.}
Although $\Theta_\tau$ is monotone and idempotent, it is not intended as a
deductive closure operator but as a \emph{normative commitment projection}:
it upgrades scores above $\tau$ to commitment without performing any
truth-theoretic or entailment-based completion.

\paragraph{Connection with derivation.}
Within the derivational framework introduced earlier, the threshold
projection $\Theta_\tau$ provides the semantic counterpart of collapse
derivation. In particular, a formula $\varphi$ becomes collapse-derivable
precisely when its valuation is mapped to commitment:

\[
\vdash_S^c \varphi
\quad \Longleftrightarrow \quad
\Theta_\tau\big(\mathrm{sc}_S(\varphi)\big)=1.
\]

Thus, collapse derivation corresponds to a projection-induced transition
in valuation space rather than to deductive closure. Logical derivation
remains distinct from normative commitment, while $\Theta_\tau$ acts as
the governance-regulated mechanism connecting epistemic dynamics to
decision.

\subsection{Semantics as Transformation}

Under the algebraic interpretation, inference becomes transformation
within valuation space:

\begin{itemize}
    \item hypotheses combine algebraically,
    \item interaction modulates composition,
    \item updates induce valuation transformations.
\end{itemize}

The process-level emphasis therefore shifts from validity alone to
stability and threshold-regulated
selection.

\subsection{Structural Consequence}

The algebraic semantics reinforces the interpretation of abductive reasoning under the $\kappa$--$\tau$ logic
as:

\begin{center}
\textbf{structured evolution within representational spaces}
\end{center}

rather than evaluation over static logical models.

\section{Example: Suspended Derivation in Crisis Reasoning}
\label{sec:example}

To illustrate the inferential dynamics captured by the
$\kappa$--$\tau$ framework, consider a simplified crisis-management
scenario involving the early stages of an epidemic outbreak.
Crisis management is a paradigmatic risk-sensitive domain:
decisions must be taken under streaming, incomplete, and partially
contradictory evidence, the consequences of premature commitment
are severe and often irreversible (misallocated public health
resources, wrongly targeted interventions, unnecessary closures),
and the asymmetry between the cost of continued investigation
and the cost of acting on the wrong model is extreme.

\subsection{Initial Hypotheses}

Suppose domain experts identify several plausible explanatory
hypotheses:

\begin{itemize}
    \item $H_1$: transmission dominated by transport hubs
    \item $H_2$: transmission dominated by schools
    \item $H_3$: transmission dominated by mass events
    \item $H_4$: transmission dominated by workplaces
\end{itemize}

Rather than selecting a single explanation, the system assigns
initial epistemic valuations:

\[
w(H_1)=0.45,\quad
w(H_2)=0.40,\quad
w(H_3)=0.35,\quad
w(H_4)=0.42
\]

No hypothesis exceeds the normative threshold $\tau=0.85$.
The threshold is deliberately high, reflecting the risk profile
of the domain: committing public health resources to a single
transmission model carries severe downside if that model is wrong.

Thus:

\[
\vdash_S^p H_i \quad \text{for all } i
\]

and no collapse derivation occurs.

\subsection{Epistemic Interaction}

Experts specify interaction structure:

\[
\kappa(H_1,H_4)=+0.6,\qquad
\kappa(H_2,H_3)=+0.3,\qquad
\kappa(H_1,H_2)=-0.4,
\]
completed on the remaining pairs by
\[
\kappa(H_4,H_2)=-0.2,\qquad
\kappa(H_1,H_3)=+0.1,\qquad
\kappa(H_4,H_3)=0,
\]
so that the interaction relation is fully specified---as the
rival-sensitive commitment check below will require. The governance
parameters of the scenario are $\tau=0.85$, $\epsilon=0.25$, and the
margin function
\[
\delta(\tau,x) = r\,\tau\,x \quad\text{with } r=0.5,
\]
a concrete instance of Definition~\ref{def:rival-sensitive}:
non-decreasing in both arguments, vanishing when rivalry vanishes.

Constructive interactions allow synthesis. Using $\lambda=1$:

\[
\mathrm{sc}_S(H_1 \otimes H_4)
=
\max(0.45,\,0.42)
+ (0.6)(0.45)(0.42)
=
0.45 + 0.1134
=
0.5634,
\]

strictly above the stronger component yet still well below $\tau$,
yielding an emergent composite explanation:

\[
H_{1,4}: \text{mobility--workplace coupling}
\]

The risk-management significance is immediate: an intervention
strategy targeting transport hubs alone or workplaces alone would
fail to address the coupled mechanism the composite represents.
(The logic contains no intervention utility model, so it does not
establish suboptimality in a decision-theoretic sense; what it
establishes is that neither single-factor account carries the
composite's explanatory warrant.) The composite hypothesis $H_{1,4}$ identifies
the \emph{interaction} between commuting patterns and workplace
density as the operative mechanism, suggesting coordinated
interventions (e.g., staggered schedules, ventilation mandates
at transit-connected workplaces) that neither single-factor
model would generate.

This composite hypothesis becomes epistemically plausible:

\[
\vdash_S^p (H_1 \otimes H_4)
\]

\subsection{Streaming Observations}

As observations arrive (mobility data, localized outbreaks,
wastewater signals), reinforcing bundles raise the weights of the
constituent hypotheses, and the composite score is recomputed at each
step by the scoring semantics with $\kappa(H_1,H_4)=0.6$ and
$\lambda=1$ held fixed:

\[
\begin{array}{c|ccc}
 & w(H_1) & w(H_4) & \mathrm{sc}_{S_t}(H_1\otimes H_4)\\
\hline
S_0 & 0.45 & 0.42 & 0.5634\\
S_1 & 0.52 & 0.48 & 0.6698\\
S_2 & 0.58 & 0.54 & 0.7679
\end{array}
\]

The rise is guaranteed by Proposition~\ref{prop:monotone-update},
since the bundles reinforce both constituents and the term
$H_1\otimes H_4$, with atomic operands and non-negative interaction,
is hereditarily constructive. The derivation remains suspended:

\[
\vdash_{S_t}^p (H_1 \otimes H_4)
\]

reflecting increasing coherence without enforced closure.
During this interval, the system is not idle: suspended
derivation actively informs provisional resource allocation
(preparing interventions consistent with the leading composite
hypothesis while maintaining fallback capacity for alternatives),
precisely the risk-hedging posture that premature collapse
would foreclose.

\subsection{Normative Collapse}

A further reinforcing bundle brings the constituent weights to
$w(H_1)=0.66$ and $w(H_4)=0.61$, whence

\[
\mathrm{sc}_{S_3}(H_1 \otimes H_4)
=
\max(0.66,\,0.61)
+ (0.6)(0.66)(0.61)
=
0.66 + 0.2416
=
0.9016
\ge \tau,
\]

triggering:

\[
\vdash_{S_3}^c (H_1 \otimes H_4)
\]

Under the governance policy, threshold satisfaction must additionally
clear the rivalry structure (Definition~\ref{def:rival-sensitive}),
and the check can now be carried out explicitly. The reinforcing
bundles were neutral towards schools and mass events
($\kappa_o = 0$ on those hypotheses), so at $S_3$ the remaining
weights stand at their initial values, $w(H_2)=0.40$ and
$w(H_3)=0.35$; both exceed $\epsilon=0.25$ and are active. Lifted
interactions with the composite:
\[
\kappa^*(H_1 \otimes H_4,\, H_2)
= \tfrac{1}{2}\big(\kappa(H_1,H_2)+\kappa(H_4,H_2)\big)
= \tfrac{1}{2}(-0.4-0.2) = -0.3,
\]
\[
\kappa^*(H_1 \otimes H_4,\, H_3)
= \tfrac{1}{2}\big(\kappa(H_1,H_3)+\kappa(H_4,H_3)\big)
= \tfrac{1}{2}(+0.1+0) = +0.05.
\]
$H_3$ is active but not a rival ($\kappa^* \ge 0$); $H_2$ is the
sole active rival. The required margin is
\[
\delta(\tau, -\kappa^*) = (0.5)(0.85)(0.3) = 0.1275,
\]
and the actual separation is
\[
\mathrm{sc}_{S_3}(H_1\otimes H_4) - \mathrm{sc}_{S_3}(H_2)
= 0.9016 - 0.40 = 0.5016 \;\ge\; 0.1275 .
\]
The margin condition is satisfied with room to spare, and commitment
proceeds: $\mathrm{Commit}_{S_3}(H_1\otimes H_4)$. Had $H_2$
instead stood at $0.80$---active, rival, within
$0.1275$ of the composite---threshold satisfaction would have held
while commitment was withheld: exactly the wedge between
$\vdash^c$ and $\mathrm{Commit}$ that rival-sensitive governance
introduces.

Collapse corresponds to commitment rather than truth-finality. On the
intended reading, crossing $\tau$ marks the point at which the
domain's governance deems the risk of acting on the composite
hypothesis lower than the risk of continued suspension; as
emphasized in Section~\ref{sec:language}, that assessment is
supplied to the logic through $\tau$ and $\delta$, not derived by
it---the threshold is the formal locus at which it enters.

The transition to commitment is thus induced by valuation dynamics
rather than deductive entailment.

Figure~\ref{fig:collapse-phase-space} illustrates the geometry of
threshold crossing. Interaction modulates the shape of the collapse
region rather than merely shifting scalar valuations: constructive
interaction ($\kappa^*>0$) \emph{lowers} the barrier to commitment,
expanding the collapse region, so that hypotheses that reinforce each
other become commit-worthy sooner; destructive interaction
($\kappa^*<0$) \emph{raises} the barrier, shrinking the collapse
region, so that hypotheses that inhibit each other require
individually stronger evidence before their synthesis warrants
commitment. The geometry thus visualizes how interaction modulates
risk tolerance---the system's willingness to commit depends not only
on how plausible each hypothesis is, but on how the hypotheses relate
to one another. Note that the plot displays threshold crossing
($\vdash^c$) only: full commitment ($\mathrm{Commit}$,
Definition~\ref{def:rival-sensitive}) has additional dimensions---the
scores and interactions of active rivals---that no two-dimensional
phase plot can show.

\begin{figure}[!htb]
\centering
\begin{tikzpicture}
\begin{axis}[
    width=0.85\linewidth,
    height=0.6\linewidth,
    xlabel={$a=\mathrm{sc}_S(t_1)$},
    ylabel={$b=\mathrm{sc}_S(t_2)$},
    xmin=0, xmax=1,
    ymin=0, ymax=1,
    samples=300,
    domain=0:0.999,
    restrict y to domain=0:1,
    unbounded coords=discard,
    grid=both,
    grid style={draw=gray!20},
    major grid style={draw=gray!35},
    legend style={
        at={(0.5,1.08)},
        anchor=south,
        draw=none,
        fill=none
    },
    legend columns=3,
]

\def\tau{0.85}

\addplot[name path=top, draw=none, forget plot] {1};

\addplot[name path=k0, thick]
{\tau/(1 + 0*x)};

\addplot[name path=kp, thick, dashed]
{\tau/(1 + 0.6*x)};

\addplot[name path=kn, thick, dotted]
{\tau/(1 - 0.6*x)};

\addplot[
    fill=gray!15,
    draw=none
] fill between[of=top and k0];

\legend{
$\kappa^*=0$,
$\kappa^*=+0.6$,
$\kappa^*=-0.6$
}

\node[font=\small] at (axis cs:0.22,0.18) {Suspended region};
\node[font=\small] at (axis cs:0.67,0.92) {Collapse region};

\end{axis}
\end{tikzpicture}
\caption{Geometry of threshold crossing (individual
commit-worthiness, $\vdash^c$) in the $(a,b)$ plane. Curves show the
boundary $\max(a,b)+\lambda\kappa^*ab=\tau$ for $\lambda=1$,
$\tau=0.85$, along the upper branch $b\ge a$ (the lower branch
follows by symmetry): suspended derivation below each curve,
threshold crossing above it. The shaded band is the crossing region
for the baseline $\kappa^*=0$; constructive interaction (dashed)
lowers the boundary, destructive interaction (dotted) raises it.
Full commitment ($\mathrm{Commit}$) has additional rival-dependent
dimensions not displayed here; see the text.}
\label{fig:collapse-phase-space}
\end{figure}
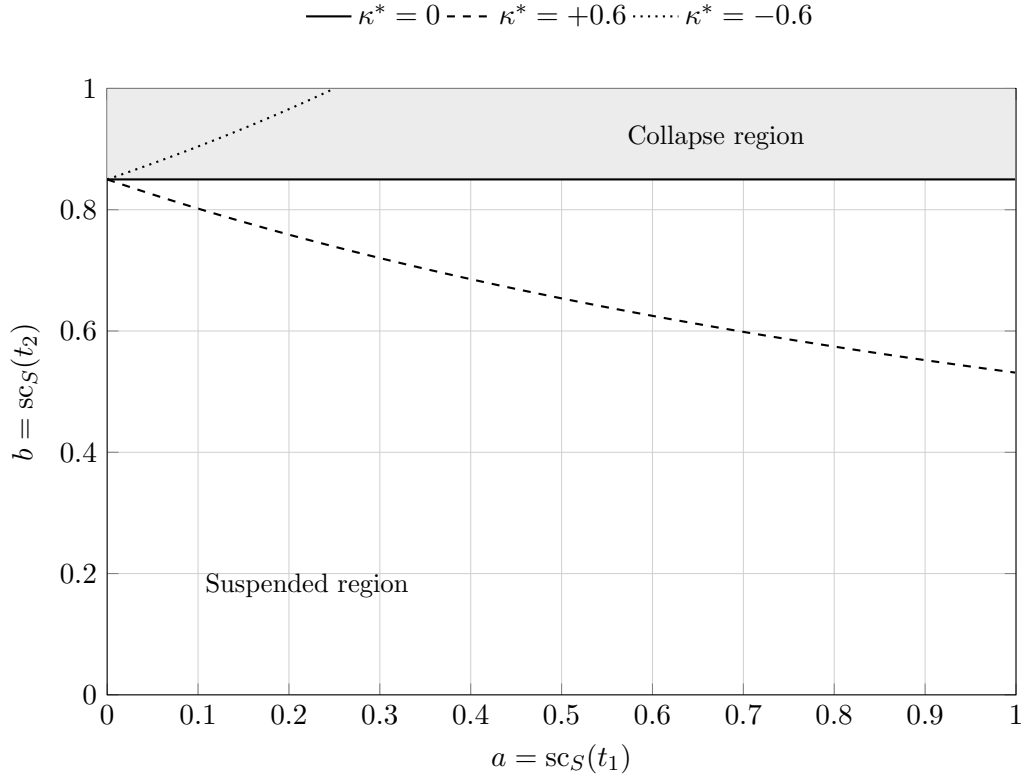

\subsection{Interpretation}

The example illustrates how the $\kappa$--$\tau$ framework
organizes inference in a risk-sensitive domain. Four structural
features are operative:

\begin{itemize}
    \item \emph{Coexistence of competing hypotheses} prevents
    premature resource commitment to a single transmission model,
    maintaining the decision-maker's option value---the ability
    to redirect interventions as evidence accumulates.
    \item \emph{Interaction-driven synthesis} generates composite
    explanations that single-hypothesis frameworks cannot represent,
    enabling intervention strategies calibrated to multi-factor
    causal structures rather than isolated variables.
    \item \emph{Suspended derivation under uncertainty} provides a
    formally sanctioned epistemic posture between ignorance and
    commitment---the system has strong inferential content (plausible
    composite hypotheses with quantified interaction) without
    the premature closure that would expose the decision-maker
    to catastrophic downside risk.
    \item \emph{Threshold-regulated closure} ensures that
    commitment is a governed transition whose timing reflects
    the risk profile of the domain: high-stakes, irreversibility-sensitive
    contexts warrant higher $\tau$, deferring collapse until
    evidence strength matches the severity of potential error.
\end{itemize}

The progression from initial hypotheses through synthesis to
eventual collapse thus models not merely an epistemic trajectory
but a \emph{risk-management} trajectory, in which the logic
itself encodes the normative discipline of deferring irreversible
decisions until the evidence warrants them. Inference progresses
through structured valuation evolution rather than eliminative
selection.

The scenario is resumed in Section~\ref{subsec:analytic-example},
where the same interaction structure is run in the opposite,
analytic direction: from the observed outbreak pattern to the
latent factor decomposition that accounts for it.

\section{The Analytic Mode: Governed Causal Decomposition}
\label{sec:analytic}

The logic has so far been developed in its \emph{synthetic} mode:
atomic hypotheses are composed upward into emergent explanations via
$\otimes$, and governance regulates when a composite becomes
commit-worthy. Abduction, however, also runs in the opposite
direction. Confronted with a complex observed state of affairs, the
reasoner asks which latent factors, in which combination and under
which internal interaction structure, account for the observed
complexity. This is the \emph{analytic} mode, developed
computationally and methodologically in companion work
\cite{pareschi2026analytic}; the present section isolates its logical
kernel---in a deliberately \emph{relationally abstracted} form, made
precise below---and shows that it is expressible within the semantic
apparatus already in place, requiring few new primitives beyond a
structured reading of the explanandum. The two modes share the
interaction relation $\kappa$ and the full governance apparatus
($\epsilon$, $\tau$, $\Pi$); they instantiate, as
Section~\ref{subsec:analytic-governance} makes precise, a common
$\kappa$--$\tau$ governance \emph{schema} over two different
valuation bases.

The analytic direction inherits the framework's central structural
advantage. Factor decomposition under uncertainty is precisely a
domain where premature commitment is dangerous: adopting the wrong
causal decomposition does not merely produce a false explanation but
directs intervention resources towards the wrong targets. The
$\tau$-governed deferral of commitment thus functions as a safeguard
against \emph{causal misattribution risk}---a risk the worked example
below renders formally visible.

\subsection{Structured Explananda and Causal Clusters}

In the synthetic mode, observations enter individually, each with
maximal valuation. In the analytic mode, the object of explanation is
not a single observation but a structured complex of observed
aspects, weighted by salience.

\begin{definition}[Structured explanandum]
\label{def:explanandum}
A \emph{structured explanandum} over an abductive state $S$ is a pair
$E = \langle A, s \rangle$ where $A \subseteq \mathsf{O}$ is a finite
nonempty set of observed \emph{aspects} and $s : A \to (0,1]$ assigns
each aspect a salience weight.
\end{definition}

Aspects are accepted observations in the sense of
Section~\ref{sec:scoring-semantics}---each carries maximal valuation
$\mathrm{sc}_S(e)=1$---but salience records their differential
importance for the decomposition problem: not everything observed
matters equally to what demands explanation.

\paragraph{Scope: a relational abstraction.}
The pair $\langle A, s \rangle$ deliberately abstracts from
structure that the full analytic framework
\cite{pareschi2026analytic} retains: there, explananda carry
temporal, spatial, causal, and operational \emph{relations} among
observed elements, and relational configuration is what allows
familiar factors arranged in a novel observed structure to constitute
a genuinely new explanandum. The present kernel treats aspects
extensionally, as a weighted set, and therefore formalizes the
relationally abstracted fragment of the analytic mode: sufficient for
the governance, interaction, and commitment structure that is this
paper's subject, but not a complete formalization of the companion
framework. Restoring a relational component
$E = \langle A, R, s \rangle$, and with it structural novelty at
the explanandum level, is registered among the future directions
(Section~\ref{sec:future-directions}).

The candidate explanatory material is a distinguished set of
\emph{latent factors} $\mathsf{F} \subseteq \mathsf{H}$: hypotheses
read causally, as factors whose operation would account for aspects
of the explanandum. The channel connecting factors to aspects is the
evidence-compatibility map
$\kappa_o : \mathsf{O} \times \mathsf{H} \to [-1,1]$ already
introduced for the dynamics of suspended derivation
(Section~\ref{sec:derivation}), here read in the inverse direction:
$\kappa_o(e,f) > 0$ records how far the operation of factor $f$ would
account for aspect $e$, while $\kappa_o(e,f) < 0$ records that $f$'s
operation would tell \emph{against} the presence of $e$. The same map
thus serves evidence-driven updating in the synthetic mode and factor
projection in the analytic mode.

\begin{definition}[Causal cluster]
\label{def:cluster}
A \emph{causal cluster} over $S$ is a triple
$C = \langle F_C, \rho_C, \kappa_C \rangle$ where
$F_C \subseteq \mathsf{F}$ is a
finite nonempty set of factors, $\rho_C : F_C \to (0,1]$ assigns each
participating factor a \emph{participation weight}, and
$\kappa_C : F_C \times F_C \to [-1,1]$, symmetric with null diagonal
like the global relation it locally overrides, is the
\emph{asserted internal interaction structure} of the decomposition.
When a decomposition does not assert a deviant organization,
$\kappa_C$ defaults to the restriction $\kappa{\restriction}_{F_C
\times F_C}$ of the global interaction relation, which functions as
a shared prior. (Symmetry of $\kappa_C$ carries the compatibility
reading of Section~\ref{sec:language} into the analytic mode: the
sign of $\kappa_C$ records whether two factors cohere as parts of
one causal account, not the direction of an influence, so a
mechanism in which one factor causally inhibits another may still
carry positive $\kappa_C$; a directed variant is deferred with the
directed global relation.)
\end{definition}

The asserted $\kappa_C$ is an independently supplied or elicited
component of the candidate decomposition---domain knowledge,
mechanistic modelling, expert judgment---not a free parameter to be
optimized against the decomposition score: a candidate that awarded
itself a coherence bonus by asserting strongly positive internal
interaction would be gaming the valuation, not proposing a
hypothesis. Validating an asserted interaction structure against
relational evidence requires the richer explanandum
$E = \langle A, R, s \rangle$ deferred to
Section~\ref{sec:future-directions}; within the relationally
abstracted kernel, the discipline on $\kappa_C$ is provenance, not
semantics.

A causal cluster is thus a structured object recording which latent
factors participate, with what weights, and under what asserted
internal organization---the analytic counterpart of a hypothesis
term. Where a term $t$ carries its explanatory content in its
compositional syntax, a cluster carries it in its weighted, organized
factor profile.

Two typing distinctions deserve emphasis. First, participation
weights are not state weights: the global $w(f)$ records the
epistemic activation of $f$ \emph{as a hypothesis} in the state,
whereas $\rho_C(f)$ records the strength of $f$'s asserted
participation \emph{in this particular decomposition}; the same
factor may carry different participation weights in rival clusters,
and a cluster's profile is part of the decomposition proposal, not of
the state. Second, because $\kappa_C$ belongs to the cluster, two
clusters may share the same factor set and the same participation
profile and yet differ in $\kappa_C$, positing different causal
organizations of the same material; they then receive different
decomposition scores through the coherence term defined below. This
is the formal basis of \emph{structural} (as opposed to material)
novelty in decomposition---the same familiar factors claimed to
interact in a new way---a distinction central to the companion
framework.

\subsection{Decomposition Scoring}

The valuation of a cluster against an explanandum has two components:
how well the weighted factors \emph{cover} the salient aspects, and
how \emph{coherent} the cluster is internally.

\begin{definition}[Coverage and fit]
\label{def:fit}
For a cluster $C=\langle F_C,\rho_C,\kappa_C\rangle$ and aspect
$e \in A$, the
\emph{coverage} of $e$ by $C$ is
\[
\mathrm{cov}_C(e) =
\Big[\, \sum_{f \in F_C} \rho_C(f)\,\kappa_o(e,f) \,\Big]_0^1 ,
\]
and the \emph{fit} of $C$ against $E=\langle A,s\rangle$ is the
salience-weighted mean coverage
\[
\mathrm{fit}_S(C \triangleright E)
=
\frac{\sum_{e \in A} s(e)\,\mathrm{cov}_C(e)}
     {\sum_{e \in A} s(e)} .
\]
\end{definition}

Negative compatibilities subtract inside the clamp: a factor whose
operation would tell against an observed aspect erodes the cluster's
coverage of that aspect. This is the formal locus of disconfirmation
in the analytic mode.

\begin{definition}[Internal interaction mass]
\label{def:iota}
For a cluster $C=\langle F_C,\rho_C,\kappa_C\rangle$ with
$|F_C|\ge 2$, the
\emph{internal interaction mass} is
\[
\iota_S(C)
=
\frac{1}{N_C}
\sum_{\{f,g\} \subseteq F_C,\, f \neq g}
\rho_C(f)\,\rho_C(g)\,\kappa_C(f,g),
\]
where $N_C = \binom{|F_C|}{2}$ ranges over unordered pairs of distinct
factors---well defined because $\kappa_C$ is symmetric, so each
unordered pair carries a unique value. For singleton clusters,
$\iota_S(C)=0$, consistently with the
null diagonal of $\kappa_C$.
\end{definition}

\begin{definition}[Decomposition score]
\label{def:decomposition-score}
The \emph{decomposition score} of $C$ against $E$ is
\[
\mathrm{sc}_S(C \triangleright E)
=
\Big[\,
\mathrm{fit}_S(C \triangleright E)
+
\lambda\, \iota_S(C)
\,\Big]_0^1 ,
\]
with $\lambda \in (0,1]$ the interaction-scaling parameter of the
state.
\end{definition}

The design principle is the same as for synthesis: a baseline
valuation perturbed by $\lambda$-scaled interaction. In the synthetic
clause, the baseline is the lattice join $\max(a,b)$ and the
perturbation is $\lambda\kappa^*ab$; in the analytic clause, the
baseline is the fit and the perturbation is the internal interaction
mass. Internally coherent clusters---whose factors reinforce one
another---score above their raw fit, capturing the intuition that a
decomposition into mutually supporting factors is explanatorily
stronger than the sum of its coverages; internally conflicted
clusters are attenuated below their fit. The correspondence is exact
in the two-factor case:

\begin{proposition}[Structural correspondence of the two modes]
\label{prop:mode-correspondence}
For a two-factor cluster
$C = \langle \{f,g\}, \rho_C, \kappa_C \rangle$ with
$\rho_C(f)=a$, $\rho_C(g)=b$,
\[
\mathrm{sc}_S(C \triangleright E)
=
\big[\, \mathrm{fit}_S(C \triangleright E) + \lambda\,\kappa_C(f,g)\,ab \,\big]_0^1 ,
\]
which has the shape of the synthesis clause of
Section~\ref{sec:scoring-semantics} with the lattice-join baseline
replaced by the fit baseline and component scores replaced by
participation weights.
\end{proposition}

\begin{proof}
Immediate from Definitions~\ref{def:iota}
and~\ref{def:decomposition-score} with $N_C=1$.
\qed
\end{proof}

The analytic analogues of the basic algebraic results of
Section~\ref{sec:algebraic-semantics} follow directly.

\begin{proposition}[Fit reduction]
\label{prop:fit-reduction}
If $\kappa_C$ vanishes on $F_C \times F_C$, then
$\mathrm{sc}_S(C \triangleright E) = \mathrm{fit}_S(C \triangleright E)$.
\end{proposition}

\begin{proof}
All summands of $\iota_S(C)$ vanish, and clamping is inert on
$\mathrm{fit}\in[0,1]$.
\qed
\end{proof}

\begin{proposition}[Bounded coherence perturbation]
\label{prop:bounded-coherence}
\[
\big|\,
\mathrm{sc}_S(C \triangleright E)
-
\mathrm{fit}_S(C \triangleright E)
\,\big|
\;\le\;
\frac{\lambda}{N_C}
\sum_{\{f,g\} \subseteq F_C,\, f\neq g}
\rho_C(f)\,\rho_C(g)\,|\kappa_C(f,g)|
\;\le\; \lambda .
\]
Moreover, for fixed fit and weights, the unclamped decomposition score
is monotone non-decreasing in each $\kappa_C(f,g)$.
\end{proposition}

\begin{proof}
The unclamped deviation equals $\lambda\,\iota_S(C)$; the triangle
inequality and $\rho_C(f)\rho_C(g)\le 1$, $|\kappa_C|\le 1$ give the
bounds,
and clamping cannot increase deviation. Monotonicity follows from
$\partial(\mathrm{fit}+\lambda\iota)/\partial\kappa_C(f,g)
=\lambda\, \rho_C(f)\rho_C(g)/N_C \ge 0$.
\qed
\end{proof}

Fit reduction is the analytic counterpart of Join Reduction
(Theorem~\ref{thm:kappa-free-reduction}): interaction-free
decomposition collapses to pure coverage, and everything the
framework adds over a weighted covering model resides in the
$\kappa$-perturbation.

\paragraph{Remark (accumulation and cluster size).}
The two baselines differ in one behaviour that should be stated
rather than obscured. The synthetic baseline is the idempotent join,
so a composite cannot outscore its best component except through
genuine positive interaction: commitment cannot be manufactured by
pooling weak, unrelated hypotheses
(Section~\ref{sec:algebraic-semantics}). The analytic baseline is a
clamped weighted sum: participation weights are not normalized, no
complexity penalty is imposed, and coverage of an aspect can
therefore be driven towards saturation by accumulating many mildly
compatible factors, even with $\kappa_C \equiv 0$. The analytic
mode thus inherits the accumulation behaviour of weighted covering
models, and the no-manufactured-commitment guarantee is a property
of the synthesis clause, not of coverage. This is a deliberate
modelling choice---set-covering accumulation is often the right
behaviour for causal decomposition, where many small contributing
factors may jointly account for an outcome---but it introduces a
size bias that governance must police: parsimony penalties,
constraints on total participation mass
$\sum_f \rho_C(f)$, or normalized coverage are natural refinements,
registered in Section~\ref{sec:future-directions}.

\subsection{Inter-Cluster Interaction}

Multiple candidate decompositions characteristically coexist under
suspended derivation. A second level of interaction then arises: a
relation measuring whether two proposed decompositions are compatible
or competing.

\begin{definition}[Inter-cluster interaction]
\label{def:kappa-star-star}
For clusters $C_1=\langle F_1,\rho_1,\kappa_{C_1}\rangle$ and
$C_2=\langle F_2,\rho_2,\kappa_{C_2}\rangle$,
\[
\kappa^{**}(C_1,C_2)
=
\frac{\displaystyle
\sum_{f \in F_1} \sum_{g \in F_2}
\rho_1(f)\,\rho_2(g)\,\kappa(f,g)}
{\displaystyle
\Big(\sum_{f \in F_1} \rho_1(f)\Big)
\Big(\sum_{g \in F_2} \rho_2(g)\Big)} ,
\]
with the diagonal convention $\kappa^{**}(C,C)=0$ extended from
Section~\ref{sec:language} (under identity of clusters), and with
$\kappa(f,f)=0$ neutralizing shared factors in the double sum.
\end{definition}

Note that the double sum runs over the \emph{global} interaction
relation $\kappa$, not over the clusters' asserted structures: a
cluster's $\kappa_C$ is its own hypothesis about how its factors are
internally organized, and is scored through $\iota$, whereas the
compatibility of two rival decompositions must be assessed against
the shared prior---otherwise each decomposition could unilaterally
legislate its relation to its competitors.

$\kappa^{**}$ is a weight-sensitive generalization of the term-level
lifting $\kappa^*$; the two levels cohere exactly:

\begin{proposition}[Level coherence]
\label{prop:level-coherence}
Let $C_1, C_2$ be clusters with $F_1 \neq F_2$ and constant
participation profiles $\rho_1, \rho_2$. Then
$\kappa^{**}(C_1,C_2) = \kappa^*(t_1,t_2)$ for any hypothesis terms
$t_1,t_2$ with $\mathrm{At}(t_i)=F_i$.
\end{proposition}

\begin{proof}
Since $F_1 \neq F_2$, no terms with these atom sets are
$\equiv_\otimes$-equivalent and the clusters are not identical, so
neither null convention applies and both sides are computed by their
averaging formulas. With $\rho_i \equiv c_i$, numerator and
denominator carry the common
factor $c_1 c_2$, which cancels, leaving the unweighted average
$\frac{1}{|F_1||F_2|}\sum\sum \kappa(f,g)$, which is the definition of
$\kappa^*$ on the atom sets. Shared factors contribute the same
diagonal zeros to both expressions.
\qed
\end{proof}

\paragraph{Remark (necessity of the restriction).}
The restriction to $F_1 \neq F_2$ is not decorative. The framework
deliberately admits distinct clusters over the \emph{same} factor
set---same material, different asserted organizations
$\kappa_{C_1} \neq \kappa_{C_2}$---and for such a pair the two
sides of the equation are governed by different conventions:
$\kappa^{**}(C_1,C_2)$ is the weighted average of global $\kappa$
over $F \times F$, generally nonzero, while term-level values over a
shared atom set depend on the $\equiv_\otimes$-structure of the
chosen terms. No uniform identity can hold there, and none should:
two organizations of the same material are exactly the case in which
cluster-level and term-level structure come apart.

\paragraph{Remark (material vs.\ structural rivalry).}
$\kappa^{**}$ as defined is \emph{material}: it measures the
compatibility of the factor bases under the shared prior, and is
blind to disagreement between the asserted structures themselves.
Two clusters over the same factors asserting incompatible internal
organizations receive whatever $\kappa^{**}$ the global $\kappa$
assigns---possibly positive---although they are rivals in an evident
structural sense. A refined rivalry valuation would decompose into a
material component $\kappa^{**}_{\mathrm{mat}}$ (the present
definition) and a structural component
$\kappa^{**}_{\mathrm{struct}}$ measuring divergence between
$\kappa_{C_1}$ and $\kappa_{C_2}$ on shared pairs, with governance
sensitive to both. The minimal framework records only the material
component; the structural refinement is registered in
Section~\ref{sec:future-directions}.

Since $|\kappa|\le 1$ and weights are positive,
$\kappa^{**}(C_1,C_2)\in[-1,1]$: the inter-cluster relation lives in
the same range as atomic interaction and admits the same reading.
Positive $\kappa^{**}$ marks decompositions that are jointly
tenable---candidates for merger or joint refinement; negative
$\kappa^{**}$ marks genuinely competing accounts of the same
explanandum, whose rivalry the governance layer must adjudicate.

\subsection{Two-Level Governance}
\label{subsec:analytic-governance}

Satisfaction extends to decomposition claims exactly as for content
formulas; decomposition judgments $P(C \triangleright E)$ and
$C_\tau(C \triangleright E)$ are the analytic instances of the
judgment sort of Section~\ref{sec:language}.
Relative to the activation floor and threshold carried by $S$:
\[
S \Vdash P(C \triangleright E)
\iff
\mathrm{sc}_S(C \triangleright E) \ge \epsilon,
\qquad
S \Vdash C_{\tau}(C \triangleright E)
\iff
\mathrm{sc}_S(C \triangleright E) \ge \tau,
\]
and the two derivation modes transfer verbatim:
\[
\vdash_S^p (C \triangleright E)
\iff
\epsilon \le \mathrm{sc}_S(C \triangleright E) < \tau,
\qquad
\vdash_S^c (C \triangleright E)
\iff
\mathrm{sc}_S(C \triangleright E) \ge \tau .
\]
A state in which several clusters satisfy
$\vdash_S^p(C_i \triangleright E)$, typically with pairwise negative
$\kappa^{**}$, is a state of \emph{suspended decomposition}: the
analytic counterpart of suspended derivation, and equally a
legitimate, analytically productive inferential output rather than a
failure to decide.

Governance operates at two levels. At the \emph{cluster level}, the
rival-sensitive commitment of
Definition~\ref{def:rival-sensitive} carries over with
$\kappa^{**}$ in the role of the lifted interaction:
\[
\mathrm{Commit}_S(C \triangleright E)
\iff
\vdash_S^c (C \triangleright E)
\ \text{and, for every rival } C' \text{ with }
\mathrm{sc}_S(C' \triangleright E) \ge \epsilon
\text{ and } \kappa^{**}(C,C') < 0,
\]
\[
\mathrm{sc}_S(C \triangleright E)
- \mathrm{sc}_S(C' \triangleright E)
\;\ge\;
\delta\big(\tau,\, -\kappa^{**}(C,C')\big).
\]
Inter-cluster interaction thereby directly governs commitment rather
than merely describing the rivalry structure: a decomposition that
clears $\tau$ while a strongly incompatible rival stands within the
margin remains in suspended decomposition until the rivalry is
resolved---the central protection against causal misattribution,
since committing intervention resources on the strength of a
hair's-breadth lead over an incompatible account is precisely the
premature convergence the analytic mode is designed to prevent.
First-past-$\tau$ is again the degenerate case $\delta \equiv 0$.
At the
\emph{factor level}, $\tau$ regulates commitment to an individual
factor as an intervention target: factor $f \in F_C$ is individually
commit-worthy against $E$ precisely when its singleton subcluster
already suffices,
\[
\mathrm{sc}_S\big(\langle \{f\},
\rho_C{\restriction}_{\{f\}},
\kappa_C{\restriction}_{\{f\}\times\{f\}}\rangle
\triangleright E\big) \ge \tau .
\]
The two levels can dissociate, and the dissociation is diagnostically
significant: a cluster may commit while none of its factors commits
individually, in which case the logic certifies the \emph{coupling}
of factors---not any factor alone---as the operative causal
structure. What is then formally identifiable as causal
misattribution is treating a singleton factor as the causal
explanation: no singleton carries commitment warrant. Whether an
intervention confined to one factor would nonetheless prove
effective is a further question the present semantics---which
contains no do-operator or action--outcome model---does not
adjudicate; what it withholds is the explanatory warrant such an
intervention would presuppose. The worked example below
exhibits exactly this configuration.

The dynamics of Section~\ref{sec:derivation} transfer in
\emph{form}, with one typing caveat: what evidence bundles update in
the analytic mode are participation profiles, not state weights. For
each active candidate cluster, the profile evolves by the
participation analogue of the update rule,
\[
\rho_{t+1,C}(f)
=
\Big[\, \rho_{t,C}(f)
+ \eta \sum_{o \in E_{t+1}} \sigma(o)\,\kappa_o(o,f)
\,\Big]_0^1 ,
\]
the same $\kappa_o$-mediated functional form acting on a different
object. The explanandum itself may
grow as new aspects are accepted, and the stabilisation alternatives
(commitment, disconfirmation, pre-emption, fixed point) are now read
at the cluster level, with pre-emption regulated by the
rival-sensitive condition above.

\subsection{Worked Continuation: Analytic Decomposition of the Crisis Scenario}
\label{subsec:analytic-example}

The synthetic example of Section~\ref{sec:example} composed the
emergent hypothesis $H_1 \otimes H_4$ (mobility--workplace coupling)
from atomic transmission hypotheses. The analytic mode runs the same
scenario in the opposite direction: from the observed outbreak
pattern to the latent factor structure that accounts for it. We
retain the interaction structure of Section~\ref{sec:example},
relabelling the transmission hypotheses as latent factors---$f_1$
(transport-hub transmission), $f_2$ (workplace transmission), $f_3$
(school transmission), $f_4$ (mass-event transmission)---so that
$\kappa(f_1,f_2)=+0.6$, $\kappa(f_1,f_3)=-0.4$, and
$\kappa(f_3,f_4)=+0.3$ are the core interactions specified there,
and $\kappa(f_2,f_3)=-0.2$, $\kappa(f_1,f_4)=+0.1$,
$\kappa(f_2,f_4)=0$ are the completing values already fixed for the
rival-sensitive check: the two modes run on \emph{identical}
interaction data. Throughout, $\lambda=1$, $\tau=0.85$,
$\epsilon=0.25$, and the margin function
$\delta(\tau,x)=0.5\,\tau x$, all as in the synthetic example.

\paragraph{Initial explanandum.}
Early surveillance yields two salient aspects:
$e_1$, a case surge concentrated in commuter towns ($s=1.0$), and
$e_2$, pronounced weekday periodicity of incidence ($s=0.8$).
Epidemiological assessment supplies the compatibility profile:

\[
\begin{array}{c|cccc}
\kappa_o & f_1 & f_2 & f_3 & f_4\\
\hline
e_1 & +0.8 & +0.5 & +0.4 & \phantom{+}0.0\\
e_2 & +0.6 & +0.7 & +0.7 & -0.2\\
e_3 & +0.4 & +0.8 & -0.6 & +0.2\\
e_4 & +0.1 & +0.2 & -0.8 & \phantom{+}0.0
\end{array}
\]

\noindent (rows $e_3,e_4$ become operative below). Two candidate
decompositions are proposed:
\[
C_1 = \langle \{f_1,f_2\},\, \rho(f_1)=0.5,\ \rho(f_2)=0.4 \rangle,
\qquad
C_2 = \langle \{f_3,f_4\},\, \rho(f_3)=0.6,\ \rho(f_4)=0.3 \rangle,
\]
a mobility--workplace decomposition and a gatherings-centred
decomposition (schools and mass events). Neither decomposition
asserts a deviant internal organization, so both carry the default
$\kappa_C = \kappa{\restriction}_{F_C \times F_C}$.

\paragraph{Suspended decomposition.}
Against $E_0=\langle\{e_1,e_2\},s\rangle$:
\[
\mathrm{cov}_{C_1}(e_1)=0.60,\quad
\mathrm{cov}_{C_1}(e_2)=0.58,\quad
\mathrm{fit}_S(C_1 \triangleright E_0)
=\frac{1.0(0.60)+0.8(0.58)}{1.8}=0.591,
\]
\[
\iota_S(C_1)=(0.5)(0.4)(0.6)=0.12,
\qquad
\mathrm{sc}_S(C_1 \triangleright E_0)=0.591+0.12=0.711,
\]
and for the rival,
\[
\mathrm{cov}_{C_2}(e_1)=0.24,\quad
\mathrm{cov}_{C_2}(e_2)=0.36,\quad
\mathrm{fit}_S(C_2 \triangleright E_0)=0.293,
\]
\[
\iota_S(C_2)=(0.6)(0.3)(0.3)=0.054,
\qquad
\mathrm{sc}_S(C_2 \triangleright E_0)=0.347 .
\]
Both decompositions are active but neither is commit-worthy:
\[
\vdash_S^p (C_1 \triangleright E_0),
\qquad
\vdash_S^p (C_2 \triangleright E_0),
\]
with
\[
\kappa^{**}(C_1,C_2)
=
\frac{(0.5)(0.6)(-0.4)+(0.5)(0.3)(0.1)+(0.4)(0.6)(-0.2)+0}
{(0.9)(0.9)}
= -0.189 :
\]
the decompositions are mildly competing, and the state is one of
suspended decomposition---strong inferential content, quantified
rivalry, no premature commitment of intervention resources.

\paragraph{Discriminating evidence.}
Subsequent surveillance accepts two further aspects: $e_3$, an age
distribution skewed towards working adults ($s=0.9$), and $e_4$,
absence of school-linked case clusters ($s=0.7$), enlarging the
explanandum to $E_1$. The same bundle updates participation profiles
through the participation analogue of the update rule
(Section~\ref{subsec:analytic-governance}) with
$\eta=0.25$ and $\sigma \equiv 1$:
\[
\Delta \rho(f_1)=+0.125,\quad
\Delta \rho(f_2)=+0.25,\quad
\Delta \rho(f_3)=-0.35,\quad
\Delta \rho(f_4)=+0.05,
\]
yielding
$C_1' = \langle \{f_1,f_2\}, (0.625,\,0.65)\rangle$ and
$C_2' = \langle \{f_3,f_4\}, (0.25,\,0.35)\rangle$, the default
$\kappa_C$ unchanged.

Against $E_1$, coverage for $C_1'$ is
$\mathrm{cov}(e_1)=0.825$, $\mathrm{cov}(e_2)=0.830$,
$\mathrm{cov}(e_3)=0.770$, $\mathrm{cov}(e_4)=0.193$, whence
\[
\mathrm{fit}_S(C_1' \triangleright E_1)=0.681,
\qquad
\iota_S(C_1')=(0.625)(0.65)(0.6)=0.244,
\]
\[
\mathrm{sc}_S(C_1' \triangleright E_1)=0.925 \;\ge\; \tau,
\qquad\text{triggering}\qquad
\vdash_S^c (C_1' \triangleright E_1).
\]
The rival collapses in the opposite direction: negative
compatibilities drive $\mathrm{cov}_{C_2'}(e_3)$ and
$\mathrm{cov}_{C_2'}(e_4)$ to the clamp at $0$, leaving
$\mathrm{fit}_S(C_2' \triangleright E_1)=0.054$ and
$\mathrm{sc}_S(C_2' \triangleright E_1)=0.080 < \epsilon$: the
gatherings-centred decomposition exits the active state entirely.
The rival-sensitive commitment condition of
Section~\ref{subsec:analytic-governance} is thereby satisfied
non-degenerately: at the moment of commitment, no rival remains
active ($\mathrm{sc} \ge \epsilon$), so the margin condition holds
vacuously, and $\mathrm{Commit}_S(C_1' \triangleright E_1)$ follows.
Had $C_2'$ instead stabilized just below $C_1'$ while their
$\kappa^{**}$ remained negative, commitment would have been
withheld---suspended decomposition maintained---until the separation
exceeded $\delta(\tau, -\kappa^{**})$.

\paragraph{Factor-level check.}
Within the committed decomposition, neither factor commits
individually:
\[
\mathrm{sc}_S\big(\langle\{f_1\},0.625\rangle \triangleright E_1\big)
=0.314,
\qquad
\mathrm{sc}_S\big(\langle\{f_2\},0.65\rangle \triangleright E_1\big)
=0.367,
\]
both far below $\tau$. What the logic certifies is the
\emph{coupling}: the decomposition commits at the cluster level while
blocking commitment to either single-factor account. An intervention
strategy targeting transport hubs alone or workplaces alone is
thereby formally marked as causal misattribution---the analytic
counterpart of the observation, in the synthetic example, that the
composite hypothesis $H_{1,4}$ suggests coordinated interventions
that neither single-factor model would generate.

\paragraph{Round trip.}
The committed decomposition $C_1'$ recovers, from the observed
outbreak pattern, exactly the mobility--workplace coupling that the
synthetic mode composed as the emergent hypothesis $H_1 \otimes H_4$
in Section~\ref{sec:example}. The scope of this observation should
be stated with care. It is a demonstration of \emph{illustrative
consistency}: on a shared scenario, with a shared interaction
relation and shared governance parameters, the two modes converge on
the same explanatory structure from opposite directions. It is not an
inversion theorem---the two valuation functions differ in their
baselines (lattice join over term scores; salience-weighted fit over
participation profiles), and Proposition~\ref{prop:mode-correspondence}
establishes a structural correspondence, not a semantic identity.
What the two modes demonstrably share is the governance schema: a
baseline valuation perturbed by a $\lambda$-scaled interaction term,
$\epsilon$/$\tau$-regulated judgment, suspended and collapse
derivation, and rival-sensitive commitment. A formal translation
between term scoring and cluster scoring, with conditions under which
one mode's scores determine the other's, is an open problem
registered in Section~\ref{sec:future-directions}.

\subsection{Position relative to Structural Causal Models}

The analytic mode also situates the framework relative to the
broader programme of causal reasoning. Pearl's structural causal
models \cite{pearl2009causality} presuppose a sufficiently specified
causal vocabulary---variables and structural equations, deterministic
or probabilistic---together with stable causal relationships; given
these, they support interventional and counterfactual inference of a
precision the present framework does not attempt. Analytic abduction
addresses the \emph{prior} problem: determining, under uncertainty,
which factors and which interactions deserve structural
formalization in the first place. The analytic mode
addresses the complementary regime---Knightian uncertainty, evolving
hypothesis spaces, genuinely novel situations---in which causality is
something one reasons \emph{towards} under suspended decomposition
rather than something assumed as a structural prior. On this reading,
a committed causal cluster is a natural \emph{precursor} of a
structural model: the point at which factor identity, participation,
and interaction structure have stabilized enough for structural
formalization to begin.

\section{Neurosymbolic Realization}
\label{sec:neurosymbolic}

The $\kappa$--$\tau$ logic has been presented as a pure semantics,
neutral with respect to the internal representation of hypotheses
(Section~\ref{sec:language}). This neutrality is what makes the
framework natively neurosymbolic: every epistemic input the logic
consumes is a quantity that neural components are well suited to
supply, while every normative parameter the logic enforces is a
quantity that should \emph{not} be learned. This section makes the
resulting division of labour explicit.

\paragraph{What the neural side supplies.}
The logic operates over abstract explanatory tokens equipped with
weights $w$, interactions $\kappa$, and evidence compatibilities
$\kappa_o$. None of these need be hand-authored. In the Quantum
Abduction implementation \cite{pare1}, hypotheses are
natural-language explanatory statements proposed by generative
models, embedded as vectors in a semantic space by modern sentence
encoders; pairwise interaction is estimated from embedding geometry,
and candidate syntheses of interacting hypotheses are themselves
generated by the language model and re-embedded for evaluation. The
same pipeline instantiates the analytic mode in the companion
framework \cite{pareschi2026analytic}, where latent factors, their
compatibility with observed aspects, and candidate causal clusters
are proposed and scored through the identical embedding machinery,
and an earlier empirical study \cite{ghisellini2025entangled}
demonstrated semantically estimated interaction among competing
strategic heuristics in agent-augmented reasoning. These systems demonstrate that
the parameters $w$, $\kappa$, and $\kappa_o$
are \emph{computationally estimable or supplyable} in practice from
unstructured content---through embedding geometry, heuristic
computation, or expert elicitation, depending on the deployment. They
do not yet demonstrate end-to-end learning calibrated against
reasoning outcomes, which remains among the future directions
(Section~\ref{sec:future-directions}); what they establish is that
the epistemic interface exists and is populated by exactly the
quantities neural components are suited to supply.

\paragraph{What the symbolic side enforces.}
What the neural components must not decide is when to commit. The
threshold $\tau$, the activation floor $\epsilon$, the margin
function $\delta$, and the governance
policies built over them (rival-sensitive pre-emption, two-level
commitment in the analytic mode) encode the risk posture of the
deployment domain: how much evidence strength is demanded before an
irreversible decision, at what point an explanation is active enough
to shape resource allocation, when a decomposition may direct
intervention. These are normative choices, accountable to domain
governance rather than to training data, and the logic keeps them
symbolic, inspectable, and fixed independently of the learned
components.

The resulting boundary is best stated in provenance-aware form:
\[
\begin{array}{ll}
w,\ \kappa_o: & \text{primarily learned or computed,}\\
\kappa: & \text{computed, learned, or expert-calibrated,}\\
\tau,\ \epsilon,\ \delta: & \text{institutionally governed.}
\end{array}
\]
The middle row carries a qualification that the convenient slogan
``the boundary is drawn at the $\kappa$/$\tau$ interface'' elides:
some interaction values---notably operational incompatibilities
between candidate decompositions, of the kind $\kappa^{**}$
aggregates---are normative or expert-elicited rather than purely
neural in origin. The hard boundary is therefore between
\emph{estimable epistemic structure}, however mixed its provenance,
and \emph{governed normative structure}, which is never learned: the
epistemic side $\langle \mathsf{H}, w, \kappa, \kappa_o \rangle$
is open to computational estimation, the normative side
$\langle \tau, \epsilon, \delta \rangle$ is not.

\paragraph{Architectural reading.}
In the taxonomy of neurosymbolic integration patterns
\cite{kautz2022third,garcez2023neurosymbolic,hitzler2022nesy}, the
resulting architecture is a pipelined hybrid: neural components
generate and quantify explanatory content, and a symbolic layer
performs structured, auditable inference over it. What the
$\kappa$--$\tau$ logic adds to this familiar pattern is a
\emph{semantics for the interface itself}. Suspended derivation and
suspended decomposition are exactly the states in which
neurally supplied content is held, graded, and composed without
commitment; every collapse is an explicit, threshold-certified event
leaving an audit trail of scores, interactions, and governance
parameters. The reasoning layer is thereby transparent in the sense
that motivates symbolic components in neurosymbolic systems: each
commitment is reconstructible from the valuation dynamics that
produced it, and each deferral is a formally sanctioned state rather
than a silent failure. Conversely, the logic without neural grounding
would face a knowledge-acquisition burden---eliciting $w$, $\kappa$,
and $\kappa_o$ by hand---that embedding-based estimation discharges.
The dependency runs in both directions, which is what makes the
integration neurosymbolic rather than merely modular.

Consolidation of this division of labour into a general architecture,
including end-to-end calibration of the learned parameters against
reasoning outcomes, is discussed among the future directions
(Section~\ref{sec:future-directions}).

\section{Related Work}
\label{sec:related-works}

The proposed $\kappa$--$\tau$ logic is intentionally minimal.
It does not aim to replace existing abductive frameworks,
but to formalize core structural features of the approach:

\begin{itemize}
    \item coexistence prior to collapse
    \item interaction-dependent synthesis
    \item governance-dependent commitment
\end{itemize}

Unlike probabilistic or purely logical systems,
collapse is regulated by normative thresholds rather than
forced by inference.

\paragraph{Position within abductive traditions.}
Traditional abductive approaches in AI---logic-based (e.g., Abductive
Logic Programming~\cite{kakas1992abductive}), Bayesian (e.g.,
statistical abduction~\cite{ishihata2011}), and set-covering
(e.g., parsimonious covering theory~\cite{reggia1983diagnostic})---can
each retain multiple explanations: abductive logic programs may admit
several solutions, statistical abduction maintains distributions over
explanations, and parsimonious covers are characteristically
multi-factor. What they do not generally combine is explicit
compositional cross-hypothesis interaction with an internal,
rival-sensitive commitment judgment. The $\kappa$--$\tau$ logic
supplies exactly this combination, modelling pre-collapse coexistence
and interaction-driven synthesis as first-class inferential states. The interaction structure
induced by $\kappa$ bears some resemblance to explanatory coherence
models developed in cognitive science \cite{thagard2000coherence},
where hypotheses may reinforce or inhibit one another within a network
of explanatory relations. The present framework differs in providing an
explicit logical language and a compositional synthesis operator for
representing such interactions.

\subsection{Relations to Logical Traditions and Reasoning Frameworks}

The proposed $\kappa$--$\tau$ framework does not belong to any
single established logical tradition. Its structural assumptions
depart from the dominant model-theoretic conception of logic
associated with the Fregean--Tarskian paradigm, where inference
is grounded in truth-conditional semantics and derivation is
typically tied to closure conditions. This subsection situates the
framework relative to several specific traditions and identifies
the precise structural points of contact and departure.

\paragraph{Uncertainty and graded plausibility.}

Several formal frameworks model reasoning under uncertainty using
graded, non-probabilistic valuations. Possibilistic logic
(Dubois and Prade \cite{dubois1988possibility,dubois1994possibilistic}) assigns necessity and possibility degrees to
formulas, yielding a qualitative account of uncertainty that,
like the present framework, avoids the additive constraints of probability.
However, possibilistic inference is organized around ranking and
threshold selection, and interaction between retained
hypotheses---in the compositional, score-perturbing form developed
here---is not normally represented. Spohn's ranking functions \cite{spohn2012laws} provide an
ordinal plausibility calculus with a well-developed theory of
iterated revision, but again without an interaction operator
analogous to $\kappa$. Halpern's general plausibility measures \cite{halpern2003reasoning,friedman1995plausibility}
abstract over both probabilistic and possibilistic frameworks,
offering a unifying algebraic treatment. Related ideas appear in
valuation-based systems for reasoning under uncertainty developed by
Shenoy \cite{shenoy1992valuation}, where information
is represented through valuations combined by structured operators.
The framework shares the graded valuation substrate of these approaches, but
introduces two features absent from all of them: epistemic interaction
($\kappa$) as a compositional operator on hypotheses, and a normative
threshold ($\tau$) that structurally decouples inference from
commitment.

\paragraph{Belief functions and uncommitted mass.}

Dempster-Shafer theory \cite{shafer1976mathematical,dempster1967upper} provides an alternative to Bayesian
probability in which belief functions assign mass to subsets of
the hypothesis space, and the total mass need not be committed
to specific singletons. The ``uncommitted mass'' assigned to
the full frame of discernment formally represents epistemic
suspension---an explicit refusal to distribute belief across
individual hypotheses. This resonates with the framework's treatment of
the pre-collapse state as a legitimate inferential output.
However, Dempster-Shafer theory combines evidence through
Dempster's rule of combination, which is commutative,
associative, and does not model an explicit pair-indexed explanatory
interaction relation between retained hypotheses. The synthesis operator, by contrast, is
non-associative (Theorem~\ref{thm:non-associativity}) and
interaction-dependent, capturing the context-sensitivity that
Dempster's rule cannot represent. Furthermore, Dempster-Shafer
theory does not incorporate a governance threshold: the
transition from uncommitted mass to committed belief is not
formally regulated.

\paragraph{Belief revision.}

The AGM framework \cite{alchourron1985logic,gardenfors1988knowledge} (Alchourr\'on, G\"ardenfors, and
Makinson) provides the standard account of rational belief
change through contraction, expansion, and revision operators
satisfying rationality postulates. The framework departs from AGM in a
fundamental respect: AGM revision presupposes that the belief
state is always closed---the agent holds a deductively closed
theory at every stage. The framework instead models epistemic states that
are explicitly \emph{non-closed}, with closure (collapse)
treated as a governed transition rather than a standing
structural property. The suspended derivation $\vdash^p$
has no AGM counterpart; it represents a maintained
state of justified non-commitment that AGM's closure
assumption precludes.

\paragraph{Many-valued logics and algebraic semantics.}

The scoring semantics operates over the unit interval
$[0,1]$ with lattice operations (min, max) and a product
t-norm structure, placing it within the family of many-valued
logics studied extensively by H\'ajek \cite{hajek1998metamathematics}. The base structure
$\langle [0,1], \min, \max, \oplus, \otimes_P \rangle$ is a
bounded distributive lattice expanded with the product t-norm and its
dual t-conorm (Section~\ref{sec:algebraic-semantics}). The framework's distinctive
contribution at the algebraic level is the $\kappa^*$-perturbation
of the lattice join (Theorem~\ref{thm:kappa-free-reduction}
establishes that the join is recovered at $\kappa^*=0$), placing
synthesis among the compensatory operators of
\cite{zimmermann1980latent} but with a pair-indexed compensation
coefficient. The operators $P$ and $C_\tau$ are threshold judgment
operators rather than modalities: they carry no accessibility
semantics and do not embed. They resemble the evaluation pattern of
Fitting's many-valued modal logics
\cite{fitting1991manyvalued,fitting1992manyvalued} only in that
satisfaction is determined by threshold conditions on valuations;
the Kripkean apparatus of accessibility over possible worlds has no
counterpart here.

\paragraph{Quantum cognition and interference-based models.}
The vocabulary of coexistence, interference, and collapse deliberately
echoes quantum cognition \cite{busemeyer2012quantum}, where
mathematical structures borrowed from quantum theory model context
effects and interference in human judgement, without any claim that
cognition is physically quantum-mechanical. The connection is
genealogical as well as structural: as noted in the introduction, the
$\kappa$--$\tau$ logic distils the governance mechanics of the Quantum
Abduction framework \cite{pare1}, in which hypotheses are embedded in
a semantic vector space, interaction is estimated from embedding
geometry, and superpositional coexistence with interference-aware
synthesis was exercised computationally on forensic, clinical, and
historical case studies, with a publicly available proof-of-concept
implementation; a companion empirical line demonstrated interaction
among competing strategic heuristics in agent-augmented reasoning
\cite{ghisellini2025entangled}. Relative to that programme, the
present paper isolates the logical kernel: amplitudes and
Hilbert-space projections are replaced by a graded scoring semantics
whose interaction and commitment structure can be studied
algebraically, with the geometric semantics deferred to the extension
outlined in Section~\ref{sec:future-directions}.

\paragraph{Dialectical argumentation and non-closure.}

A partial parallel exists with dialectical
proof procedures in assumption-based argumentation
(Dung \cite{dung1995acceptability}; Thang et al.
\cite{thang2022infinite}), where proof procedures may fail to
terminate and \emph{strict admissibility} is introduced to
characterize what non-terminating procedures compute. The phenomena
should not be conflated: there, the issue is procedural
non-termination of argument construction, treated as an obstacle to
be worked around; here, suspended derivation is a \emph{semantic}
condition---a stabilized valuation whose closure is deliberately
withheld---and an epistemically legitimate output. The parallel is
nonetheless instructive in one respect: both lines of work find that
the boundary between settled and unsettled inferential states is
structurally significant and demands dedicated formal machinery
rather than dismissal as pathology.

\paragraph{Bounded rationality and ecological decision-making.}

The bounded rationality tradition
(Gigerenzer \cite{gigerenzer1999simple,gigerenzer2002bounded};
Simon \cite{simon1955behavioral}, and
collaborators) shares with the present framework the diagnosis that standard
probabilistic reasoning is inadequate for real decision-making
environments---a point documented in clinical settings by
Martignon, Erickson, and Viale \cite{martignon2022transparent}, who show that physicians
systematically fail to perform correct Bayesian inference under
time pressure. However, the therapeutic responses diverge.
The bounded rationality program advocates simplification:
fast-and-frugal heuristics, lexicographic decision trees, and
\emph{satisficing} strategies---Simon's term for accepting
solutions that are good enough rather than optimal---that
exploit environmental structure through cognitive economy.
The $\kappa$--$\tau$ logic takes the complementary approach
of \emph{enriching} inference with interaction and governance,
targeting precisely the decision regimes---high-stakes,
low-frequency, irreversibility-sensitive---where heuristic
simplification would itself be dangerous. Fast-and-frugal trees
are eliminative by design: at each node, one branch exits
immediately with a classification. This architectural property,
which is a virtue in high-volume, moderate-stakes environments
(speed, transparency, low cognitive load), becomes a liability
in tail-risk domains, where it structurally accelerates
commitment and forecloses the sustained investigation that
asymmetric downside costs demand. Fast-and-frugal trees cannot represent
suspended judgment, hypothesis interaction, or governance-regulated
commitment. The two approaches are thus complementary rather
than competing, addressing different segments of the decision
landscape.

\paragraph{Neurosymbolic AI.}

The neurosymbolic programme
\cite{kautz2022third,garcez2023neurosymbolic,hitzler2022nesy} seeks
architectures in which neural components contribute learning from
unstructured data and symbolic components contribute sound,
inspectable reasoning. Much of the literature concentrates on
integrating neural learning with deductive apparatus---logic
programming, knowledge graphs, differentiable theorem proving---where
the symbolic layer certifies entailment. The $\kappa$--$\tau$ logic
targets a different integration point: \emph{abductive} reasoning
under sustained uncertainty, where the symbolic layer certifies not
entailment but governed commitment. As developed in
Section~\ref{sec:neurosymbolic}, the framework assigns the epistemic
parameters ($w$, $\kappa$, $\kappa_o$) to neural estimation and
reserves the normative parameters ($\tau$, $\epsilon$, $\delta$)
for explicit
governance, so that the estimable/governed boundary is itself a formal
feature of the logic rather than an implementation convention. Recent surveys do organize parts of the field around uncertainty
quantification, intervenability, human governance, and deferral under
detected uncertainty; what we have not identified is an existing
neurosymbolic taxonomy that treats the separation between estimable
epistemic structure and institutionally governed commitment as a
\emph{primitive of an abductive semantics}---as a formal boundary
internal to the logic, rather than a property of how neural and
symbolic computation interleave.

\paragraph{Summary.}

These comparisons identify the $\kappa$--$\tau$ framework's
distinctive position. The framework shares the graded valuation substrate
of many-valued and possibilistic logics, the non-additive
treatment of uncommitted belief found in Dempster--Shafer theory,
and the sensitivity to non-termination exhibited by dialectical
argumentation. It departs from all of these in two respects:
the interaction operator $\kappa$ introduces compositional
context-sensitivity absent from existing uncertainty frameworks,
and the normative threshold $\tau$ structurally decouples
derivation from closure in a combination---graded interaction-driven
synthesis under governed, rival-sensitive commitment---not normally
represented in standard logical or decision-theoretic frameworks.

The $\kappa$--$\tau$ logic, therefore, defines a distinct inferential regime centered on
epistemic interaction, suspended derivation, and normative closure.

\section{Future Directions}
\label{sec:future-directions}

The framework developed in this paper opens several lines of
extension, which fall broadly into three directions: structural
extensions of the logical semantics, computational integration with
learning-based systems, and conceptual extensions of abductive
reasoning itself.

\subsection*{Structural Extensions}

The algebraic semantics proposed here provides a minimal structural
foundation for abductive inference under the $\kappa$--$\tau$ framework. Several natural extensions follow.

First, geometric refinements may be introduced by embedding the
valuation structure within vector-space frameworks.
Hilbert-space-inspired semantics---in the lineage of quantum logic
\cite{birkhoff1936logic} and quantum cognition
\cite{busemeyer2012quantum}---offer one possible extension in which
hypotheses are represented as state vectors, scoring corresponds to
projections, and interaction effects correspond to operator-induced
transformations. This is not purely speculative: the Quantum Abduction
implementation \cite{pare1} already realizes a computational prototype
of this geometry, with hypotheses embedded as vectors in a semantic
space and interaction read off from embedding structure.

Second, transformational refinements may be obtained through
category-theoretic perspectives. The interaction and synthesis
structure induced by $\kappa$ and $\otimes$ suggests a potential
categorical interpretation of explanatory composition. In such
formulations, epistemic states may be modelled as objects, inference
steps as morphisms, and synthesis operators as monoidal constructions,
potentially connecting the present framework with recent categorical
approaches to compositional semantics and reasoning.

These perspectives should not be understood as alternatives to the
algebraic semantics, but as structured extensions providing additional
representational, geometric, or compositional expressiveness.

Third, a refinement of negation may decouple a hypothesis from its
complement. The present framework uses involutive negation
$\mathrm{sc}_S(\neg\varphi)=1-\mathrm{sc}_S(\varphi)$, which imposes a
duality and excludes joint plausibility of $\varphi$ and $\neg\varphi$
above $\epsilon=\tfrac12$. A natural refinement replaces the single score
with a possibility/necessity pair $(\pi,\nu)$, in the manner of
Dubois--Prade possibility theory or the belief/plausibility duality of
Dempster--Shafer, where $\nu(\varphi)\le 1-\pi(\neg\varphi)$ rather than
equality. This would let $\varphi$ and $\neg\varphi$ carry independent
non-additive support and is the natural setting in which to lift the
$\epsilon\le\tfrac12$ restriction noted in Section~\ref{sec:scoring-semantics}.

\subsection*{Learning and Calibration}

A second direction concerns the consolidation of the neurosymbolic
division of labour articulated in Section~\ref{sec:neurosymbolic}
into a general, empirically evaluated architecture. Three open
problems stand out. First, \emph{end-to-end calibration}: the
embedding-based estimation of $w$, $\kappa$, and $\kappa_o$
demonstrated in existing implementations
\cite{pare1,pareschi2026analytic} is currently feed-forward; feedback
from reasoning outcomes (committed conclusions later confirmed or
disconfirmed) should be used to refine the estimators, turning the
interaction structure into a genuinely learned quantity with error
signal. Second, \emph{governance calibration}: while $\tau$, $\epsilon$,
and the margin function $\delta$ must remain normatively
set, the mapping from domain risk profiles (asymmetry of downside
costs, reversibility of interventions) to governance parameters
$r \mapsto \langle \tau_r, \delta_r \rangle$ is itself
a modelling problem that deserves systematic treatment rather than
case-by-case stipulation. Third, \emph{benchmarking}: the framework's
claimed advantages over eliminative baselines---Bayesian model
selection, possibilistic ranking, fast-and-frugal trees---in
tail-risk regimes should be assessed on controlled decision tasks
where premature commitment carries measurable cost.

\subsection*{Completing the Analytic Mode}

Two open problems concern the relation between the two modes and the
expressiveness of the analytic kernel. First, a \emph{translation
theorem}: the round trip of Section~\ref{subsec:analytic-example}
exhibits illustrative consistency between synthesis and analysis, but
no formal translation $T$ from hypothesis terms to clusters has been
given, nor conditions under which
$\mathrm{sc}^{\mathrm{syn}}_S(t)$ and
$\mathrm{sc}^{\mathrm{an}}_S(T(t) \triangleright E)$ agree or at
least preserve order. Establishing such a result---or delimiting why
only the governance schema, and not the valuation, transfers---would
settle the exact sense in which the two modes are one logic. Second,
the \emph{relational explanandum}: restoring relational structure
$E = \langle A, R, s \rangle$ to the kernel of
Section~\ref{sec:analytic}, so that temporal order, dependency
topology, and structural novelty at the explanandum level become
representable, closing the expressiveness gap with the full analytic
framework \cite{pareschi2026analytic}.

Three further refinements were identified in the body of the paper
and are collected here. A \emph{directed interaction variant}: the
minimal framework reads $\kappa$ as symmetric compatibility
(Section~\ref{sec:language}); an asymmetric variant, in which
$\kappa(f,g)$ encodes directed influence and $\iota$ sums over
ordered pairs normalized by $|F_C|(|F_C|-1)$, would be more
expressive for causal modelling at the cost of separating the
compatibility and influence readings. \emph{Structural rivalry}: the
material $\kappa^{**}$ of Section~\ref{sec:analytic} is blind to
disagreement between asserted internal organizations; a structural
component $\kappa^{**}_{\mathrm{struct}}$ measuring divergence
between $\kappa_{C_1}$ and $\kappa_{C_2}$ on shared factor pairs,
combined with the material component in the governance condition,
would let the logic adjudicate same-material, different-structure
rivalry. \emph{Parsimony}: the analytic baseline inherits the
accumulation behaviour of weighted covering models
(Section~\ref{sec:analytic}); parsimony penalties, constraints on
total participation mass, or normalized coverage would let
governance police the resulting size bias where the domain demands
it.

\subsection*{Convergence of the Analytic and Geometric Extensions}

Finally, the analytic mode formalized in Section~\ref{sec:analytic}
opens a specifically geometric line of development: the convergence
of the causal-cluster machinery with the vector-space refinements
outlined above---causal clusters as vectors in a factor space,
inter-cluster interaction $\kappa^{**}$ through operator-induced
transformations, collapse as a projection condition. The companion
framework \cite{pareschi2026analytic} additionally develops the
protocol dimension of the analytic mode---the legibility of suspended
decomposition as a coordination object between human and artificial
agents in epidemiological crisis decomposition and adversarial cyber
threat analysis---whose full formal integration with the semantics
developed here, in particular the interaction between coordination
acts and the update dynamics, is a natural next step.

More broadly, future work may address:

\begin{itemize}
    \item richer algebraic interaction operators,
    \item dynamic update logics,
    \item proof-theoretic characterization,
    \item integration with learned representational spaces.
\end{itemize}

The $\kappa$--$\tau$ framework, therefore, opens a layered semantic program in which algebraic,
geometric and transformational structures jointly contribute to a
process-centered logic of abductive reasoning.

\section{Conclusion}
\label{sec:conclusion}
The $\kappa$--$\tau$ framework reframes abductive inference by
decoupling derivation from closure.

Rather than treating reasoning as a process culminating in the
selection of a single surviving hypothesis, the proposed logic
models inference as structured valuation dynamics in which
coexistence, interaction, and synthesis are inferentially
legitimate states.

Collapse is therefore not interpreted as a truth-theoretic
necessity, but as a normatively regulated transition governed
by epistemic thresholds.

This shift replaces eliminative reasoning architectures with a
process-centered inferential regime where unresolved tensions
are preserved, graded, and rendered analytically productive.

Within this perspective, logic functions not merely as a calculus
of validity, but as a coordination principle governing the evolution
of epistemic states.

The logic has been developed in two complementary modes sharing a
single interaction relation and a single governance apparatus. In the
synthetic mode,
hypotheses compose upward into emergent explanations; in the
analytic mode, complex explananda decompose into governed causal
clusters, with interaction and commitment operating at both the
cluster and the factor level. The two modes instantiate a common
$\kappa$--$\tau$ governance schema over different valuation
bases, and the worked crisis scenario
closes a round trip---the analytic mode recovering, from the
observed pattern, precisely the coupling that the synthetic mode
composed---exhibiting their consistency in practice, with the formal
translation between the modes registered as an open problem.

The framework is, moreover, natively neurosymbolic in its division
of labour: the epistemic parameters over which it reasons are
quantities that are computationally estimable or supplyable---through
embedding-based computation, learning, or expert
elicitation, though not yet learned end-to-end---while the normative parameters through which
it governs remain symbolic, inspectable, and accountable to domain
governance. The boundary between estimable epistemic structure and regulated
commitment---drawn, in provenance-aware form, between
$\langle w, \kappa, \kappa_o \rangle$ and
$\langle \tau, \epsilon, \delta \rangle$---rather than any
particular estimation technique, is
the framework's contribution to the neurosymbolic programme.

The risk-theoretic implications of this reframing are direct.
In domains characterized by asymmetric downside costs---where
the consequences of premature commitment vastly exceed the
costs of continued investigation---the standard inferential
imperative to resolve uncertainty as quickly as possible becomes
itself a source of risk. The $\kappa$--$\tau$ logic provides
formal machinery for resisting this imperative: the threshold
$\tau$ and the margin function $\delta$ provide the parameters
through which commitment can be calibrated to the stakes of
the decision, while the interaction operator $\kappa$ enables
the system to extract inferential value from sustained
hypothesis coexistence rather than treating it as a deficiency
to be eliminated. Suspended derivation is thus not a failure
to decide but a principled epistemic posture---one whose
risk-management value has been illustrated throughout this paper
in clinical diagnosis, crisis management, and causal
decomposition under uncertainty.

The resulting framework provides a formal foundation for
reasoning under sustained uncertainty, naturally aligning
algebraic structure and interaction-sensitive dynamics
with the normative demands of risk-sensitive decision-making.


\section*{Funding}

This work was supported by Ermete---grant number
B39J25000400005---2023 MIMIT (Ministero delle Imprese e del Made in
Italy).

\end{document}